\pdfoutput=1
\documentclass{article}
\usepackage[numbers]{natbib}

\usepackage[preprint]{neurips_2022}

\usepackage[utf8]{inputenc} % allow utf-8 input
\usepackage[T1]{fontenc}    % use 8-bit T1 fonts
\usepackage{hyperref}       % hyperlinks

\hypersetup{
    colorlinks=true,
    linkcolor=blue,
    filecolor=magenta,      
    urlcolor=blue,
    }

\usepackage{url}            % simple URL typesetting
\usepackage{booktabs}       % professional-quality tables
\usepackage{amsfonts}       % blackboard math symbols
\usepackage{nicefrac}       % compact symbols for 1/2, etc.
\usepackage{microtype}      % microtypography
\usepackage{xcolor}         % colors
\usepackage{mathtools}
\usepackage{amssymb,amsmath}
\usepackage[linesnumbered,ruled,vlined]{algorithm2e}
\usepackage{subcaption}
\usepackage{wrapfig}
\usepackage{float}
\usepackage{placeins}
\usepackage{amsthm}
\usepackage{multirow}
\usepackage{hyperref}

\newcommand{\inner}[2]{\langle #1, #2 \rangle}

\newcommand{\E}{\mathbb{E}}

\newcommand{\R}{\mathbb{R}}
\newcommand{\N}{\mathbb{N}}

\SetKwInput{KwInput}{Input}
\SetKwInput{KwOutput}{Output}

\title{Downlink Compression Improves Top$_K$ Sparsification}

\newtheorem{assumption}{Assumption}[section]
\newtheorem{theorem}{Theorem}[section]
\newtheorem{lemma}[theorem]{Lemma}

\author{%
  William Zou \\
  University of Waterloo \\
  \texttt{wyzou@uwaterloo.ca} \\
  \And
  Hans De Sterck \\
  University of Waterloo \\
  \texttt{hans.desterck@uwaterloo.ca} \\
  \And
  Jun Liu \\
  University of Waterloo \\
  \texttt{j.liu@uwaterloo.ca} \\
}

\begin{document}

\maketitle

\begin{abstract}
Training large neural networks is time consuming. To speed up the process, distributed training is often used. One of the largest bottlenecks in distributed training is communicating gradients across different nodes. Different gradient compression techniques have been proposed to alleviate the communication bottleneck, including top$_K$ gradient sparsification, which truncates the gradient to the largest $K$ components before sending it to other nodes. While some authors have investigated top$_K$ gradient sparsification in the parameter-server framework by applying top$_K$ compression in both the worker-to-server (uplink) and server-to-worker (downlink) direction, the currently accepted belief says that adding extra compression degrades the convergence of the model. We demonstrate, on the contrary, that adding downlink compression can potentially improve the performance of top$_K$ sparsification: not only does it reduce the amount of communication per step, but also, counter-intuitively, can improve the upper bound in the convergence analysis. To show this, we revisit non-convex convergence analysis of top$_K$ stochastic gradient descent (SGD) and extend it from the unidirectional to the bidirectional setting. We also remove a restriction of the previous analysis that requires unrealistically large values of $K$. We experimentally evaluate bidirectional top$_K$ SGD against unidirectional top$_K$ SGD and show that  models trained with bidirectional top$_K$ SGD will perform as well as models trained with unidirectional top$_K$ SGD while yielding significant communication benefits for large numbers of workers.
\end{abstract}

\section{Introduction}

Deep neural networks trained on large datasets are known to achieve state-of-the-art performance in accuracy. However, the large size of the models and datasets severely impacts the training time of the model. To decrease the training time, distributed machine training techniques are often used, which in recent years have involved scaling stochastic gradient descent (SGD) to multiple processes \cite{hermans2017scalable, ben2019demystifying}.

The standard way to scale SGD on multiple processes is through data parallelism where the training set is split across $n$ different processes \cite{hermans2017scalable}. Each node will calculate a stochastic gradient from their allocated data independently and in parallel. The nodes then communicate the gradient with each other before updating model parameters. We want to minimize the weighted average of $N$ different functions $(F^q)_{q=1}^{N}: \R^d \to \R$ across $N$ nodes
\begin{equation}
    \min_{w \in \R^d} F(w) \triangleq \sum_{q = 1}^N p_q  F^q(w),
\end{equation}
by applying the iteration
\begin{equation}
\label{distributed_update_step}
    w_t = w_{t-1} - \alpha_t \sum_{q = 1}^N p_q g^q(w_{t-1}, \xi_{t - 1}^q),
\end{equation}
where $w_t$ is the model parameters at iteration $t$, $g^q(\cdot, \cdot)$ is a stochastic gradient of $F$, $\alpha_t$ is the step size at time $t$, and $\xi_{t-1}^q$ represents a local sample, similar to how it is used in \cite{bottou2018optimization}.

Communication between nodes is one of the largest bottleneck in distributed machine learning \cite{xu2020compressed}. Most efforts in increasing the performance speed of distributed learning come from reducing this bottleneck. One of the most well-studied compression technique is sparsification, which focuses on reducing communication between worker nodes by sending only a sparse subset of the gradient \cite{ben2019demystifying, xu2020compressed}. The most popular of these methods is top$_K$ gradient sparsification, which truncates the gradient to the largest $K$ components by magnitude \cite{chen2021distributed,xu2020compressed}. Top$_K$ gradient sparsification was originally only used during the worker-to-server (uplink) communication of the parameter-server framework \cite{sattler2019robust}, since the gradient sent by the server is sparse if the number of participating workers is low.

Due to the increased number of workers used in distributed training, Sattler \emph{et al.} recommended adding top$_K$ gradient sparsification during server-to-worker (downlink) communication as well, and reported that sparsifying the gradients in both uplink and downlink communication (bidirectional) reduces the final accuracy by at most 3\% compared to only using uplink  \cite{sattler2019robust}. Theoretical frameworks from Tang \emph{et al.} and Zheng \emph{et al.} have been developed to analyze the convergence of bidirectionally compressed error-compensated SGD, which can be used to analyze bidirectional top$_K$ gradient sparsification \cite{tang2019doublesqueeze, zheng2019communication}. However, these studies suggest that adding downlink sparsification will degrade the convergence of the model \cite{philippenko2021preserved, tang2019doublesqueeze, zheng2019communication}. We show that under the theoretical framework provided by Alistarh \emph{et al.} \cite{alistarh2018convergence}, we can construct a upper bound for the non-convex convergence of bidirectional top$_K$ SGD that is potentially smaller than the upper bound of unidirectional top$_K$ SGD.

\textbf{Related work}. To the best of our knowledge, top$_K$ sparsification was first proposed by Aji and Heafield \cite{aji2017sparse}. Alistarh \emph{et al.} and Stich \emph{et al.} later analyzed the convergence rate of unidirectional top$_K$ sparsification, and showed that the scheme converges at the same rate as vanilla SGD \cite{alistarh2018convergence, stich2018sparsified}.

Bidirectional top$_K$ SGD has been used in the Federated Learning setting. Sattler \emph{et al.} evaluated communication protocols for Federated Learning, and suggested adding top$_K$ sparsification in the server-to-worker direction \cite{sattler2019robust}. They reported that using bidirectional top$_K$ SGD did not destabilize the training as much as using bidirectional signSGD. 

% Han \emph{et al.} reported an approach that guarantees  \cite{han2020adaptive}.

Beyond applications in Federated Learning, theoretical frameworks to evaluate the convergence rate of bidirectional biased compression techniques with error-feedback framework have been developed in \cite{tang2019doublesqueeze, zheng2019communication}. As far as we know, Tang \emph{et al.} was the first to have developed convergence analysis for bidirectional error feedback SGD compatible with any compression technique, and showed the same convergence rate as SGD under certain assumptions \cite{tang2019doublesqueeze}. Zheng \emph{et al.} developed a similar theoretical framework to Tang \emph{et al.} independently \cite{zheng2019communication}. Since then, there have been theoretical frameworks developed for variations of bidirectional algorithms \cite{liu2020double, philippenko2020bidirectional}. Liu \emph{et al.} developed DORE, which compresses and sends gradient residuals \cite{liu2020double}, and Philippenko and Dieuleveut introduced the Artemis framework to analyze unbiased compressors \cite{philippenko2020bidirectional}. To the best of our knowledge, no bidirectional compression framework has shown that adding downlink compression can improve the convergence bound of the unidirectional compression framework.

\textbf{Contributions}. To summarize, the contributions of our paper are as follows:

\begin{itemize}
    \item \textbf{Showing downlink compression can improve constants in the convergence bound of top$_K$ SGD}: We extend the convergence analysis of Alistarh \emph{et al.} to bidirectional top$_K$ SGD \cite{alistarh2018convergence}, and show that bidirectional top$_K$ has the same convergence rate as unidirectional top$_K$ SGD but with constants that may be smaller.
    
    \item \textbf{Generalization of non-convex analysis of top$_K$ SGD}: We fix a major limitation of the original unidirectional top$_K$ non-convex analysis in \cite{alistarh2018convergence} so that it works for $K < \frac{1}{2} d$, where $d$ is the size of the gradient. Note that $K\ge \frac12 d$ means that one can compress the gradient by a factor of at most 2, which is not practical. 
    
    \item \textbf{Estimate of compression factor achieved by downlink compression}: We provide an empirical estimate of the compression factor achieved by downlink sparsification, and show that top$_K$ SGD can provide significant communication benefits when the number of workers is large.
\end{itemize}

\section{\texorpdfstring{Bidirectional Top$_K$ SGD Algorithm}{Lg}}

We describe top$_K$ SGD and bidirectional top$_K$ SGD in Algorithms \ref{unidirectional_algorithm_client} and \ref{bidirectional_algorithm_client}. For both algorithms, during the $t$-th iteration, each worker $q$ calculates the stochastic gradient $g^q(w_{t-1},\xi_t^q)$, compresses the error-compensated value $\alpha_{t-1} g^q(w_{t-1},\xi_t^q) + \epsilon_{t-1}^q$, updates the local error term with $\epsilon_{t-1}^q = a_t^q - \text{Top}_K(a_t^q)$, then sends the error-compensated value to the server.

The server aggregates all error-compensated values $\sum_{q = 1}^N p_q \text{Top}_K(a_t^q)$. In unidirectional top$_K$ SGD, this values is sent directly back to each individual worker. In bidirectional top$_K$ SGD, the server compresses the error-compensated value $\sum_{q = 1}^N p_q \text{Top}_K(a_t^q) + \delta_{t-1}$, updates the global error term with $\delta_t = g_t - \text{Top}_K(g_t)$, before sending the error-compensated value back to each individual worker, where it will be used to update local models.

The main benefit of bidirectional sparsification is the communication saved in the server-to-worker direction. Note that $\sum_{q = 1}^N p_q \text{Top}_K(a_t^q)$ is a sparse vector when $N \ll D$. We investigate the communication saved in more detail in Section \ref{subsection:communication_saved}.

\begin{minipage}[t]{0.46\textwidth}
\vspace{0pt}
\begin{algorithm}[H]
\label{unidirectional_algorithm_client}
\SetAlgoLined
\KwInput{Stochastic Gradient Oracle $g^q(\cdot; \cdot)$, learning rate sequence $\{ \alpha_t \}_{t \ge 0}$, probability vector $(p_0, \ldots, p_N)$}
 Initialize $w_0 \in \R^d; \epsilon_0^q = 0 \in \R^d$; $t=1$\;
 \While{$t \ge 1$}{
  \textbf{On worker $q$}: \\
  $a_{t}^q {\leftarrow} \epsilon_{t - 1}^q + \alpha_{t - 1} g^q(w_{t - 1}, \xi_{t - 1}^q)$\;
  $\epsilon_{t}^q {\leftarrow} a_{t}^q - \text{Top}_K(a_{t}^q)$\;
  $\text{SEND}(\text{Top}_K(a_{t}^q), \text{server})$\;
  \textbf{On server}: \\
    \For{\text{every worker} $q$}{
        $\text{RECV}(\text{Top}_K(a_{t}^q), q)$\;
    }
    $g_{t} {\leftarrow} \sum_{q = 1}^{N} p_q \text{Top}_K(a_{t}^q)$\;
    \For{\text{every client} $q$}{
        $\text{SEND}(g_{t}, q)$\;
    }
  \textbf{On worker $q$}: \\
  $\text{RECV}(g_t, \text{server})$\;
  $w_{t} {\leftarrow} w_{t - 1} - g_t$\;
  $t = t + 1$\;
 }
 \caption{Unidirectional Top$_{K}$ SGD}
\end{algorithm}
\end{minipage}
\hfill
\begin{minipage}[t]{0.46\textwidth}
\begin{algorithm}[H]
\label{bidirectional_algorithm_client}
\SetAlgoLined
\KwInput{Stochastic Gradient Oracle $g^q(\cdot; \cdot)$, learning rate sequence $\{ \alpha_t \}_{t \ge 0}$, probability vector $(p_0, \ldots, p_N)$}
 Initialize $w_0 \in \R^d; \epsilon_0^q = 0 \in \R^d$; $\delta_0 = 0 \in \R^d$; $t=1$\;
 \While{$t \ge 1$}{
  \textbf{On worker $q$}: \\
  $a_{t}^q {\leftarrow} \epsilon_{t - 1}^q + \alpha_{t - 1} g^q(w_{t - 1}, \xi_{t - 1}^q)$\;
  $\epsilon_{t}^q {\leftarrow} a_{t}^q - \text{Top}_K(a_{t}^q)$\;
  $\text{SEND}(\text{Top}_K(a_{t}^q), \text{server})$\;
  \textbf{On server}: \\
  \For{\text{every worker} $q$}{
    $\text{RECV}(\text{Top}_K(a_{t}^q), q)$\;
  }
  $g_{t} {\leftarrow} \sum_{q = 1}^{N} p_q \text{Top}_K(a_{t}^q) + \delta_{t-1}$\;
  $\delta_{t} {\leftarrow} g_{t} - \text{Top}_K(g_{t})$\;
    \For{\text{every client} $q$}{
     $\text{SEND}(\text{Top}_K(g_{t}), q)$\;
  } 
  \textbf{On worker $q$}: \\
  $\text{RECV}(\text{Top}_K(g_{t}), \text{server})$\;
  $w_{t} {\leftarrow} w_{t - 1} - \text{Top}_K(g_{t})$\;
  $t = t + 1$\;
 }
 \caption{Bidirectional Top$_{K}$ SGD}
\end{algorithm}
\end{minipage}

\section{\texorpdfstring{Convergence Analysis of Bidirectional Top$_K$ SGD}{Lg}}
\label{headings}

In this section, we analyze the convergence of the stochastic sequence $\{w_t\}_{t \ge 0}$ in Algorithm \ref{bidirectional_algorithm_client} in the non-convex setting. We follow the proof structure created by Alistarh \emph{et al} \cite{alistarh2018convergence} for analyzing unidirectional top$_K$ sparsification.

\subsection{Analytic Assumptions}
\label{section:assumptions}
\begin{assumption}[\textbf{Existence of lower bound}]
\label{lower_bound_assumption}
There exists some constant $F^*$ such that $F(w) \ge F^*$ for all $w \in \R^d$.

\end{assumption}

\begin{assumption}[\textbf{Biased compressor guarantee}]

\label{biased_compressor_assumption}
There exists some $\gamma \in (0, 1)$ such that
\normalfont
\begin{align}
\label{uplink_biased}
    \| \text{Top}_K(w) - w \|_2^2 &\le (1 - \gamma) \| w \|_2^2 & \forall w \in \R^d.
\end{align}

\end{assumption}

% We mention that our analysis would not change if $K$ was different for uplink and downlink compression. Assume that  

% \begin{align}
% \label{downlink_biased}
%     \| \text{Top}_{K_1}(w) - w \|_2^2 ] &\le (1 - \gamma_1) \| w \|_2^2 & \forall w \in \R^d,
% \end{align}
% and

% \begin{align}
%     \| \text{Top}_{K_2}(w) - w \|_2^2 &\le (1 - \gamma_2) \| w \|_2^2 & \forall w \in \R^d,
% \end{align}
% then we can choose $\min\{ \gamma_1, \gamma_2 \}$ for the $\gamma$ value in Assumption \ref{biased_compressor_assumption}. Finally,

This assumption is commonly used with biased compression schemes \cite{beznosikov2020biased, karimireddy2019error}. We mention that we can always find a $\gamma$ value to satisfy this assumption for top$_K$, since we know
\begin{equation}
\label{topk_bound}
\| \text{Top}_K(w) - w  \|_2^2 \le \frac{d - K}{d} \| w \|_2^2,
\end{equation}
where $d$ is the dimension of the gradient.

\begin{assumption}[\textbf{Lipschitz continuous gradient}]
\label{L_smooth_assumption}
 $F$ is continuously differentiable and the gradient $\nabla F: \R^d \to \R^d$ is Lipschitz continuous with constant $L > 0$, i.e.,
\begin{align*}
    \| \nabla F (w) - \nabla F(v) \|_2 &\le L \|w - v \|_2  &  \forall w, v \in \R^d.
\end{align*}

\end{assumption}

\begin{assumption}[\textbf{First and second moments}]
\label{first_and_second_moment_assumption}
The objective function and Algorithm \ref{bidirectional_algorithm_client} satisfy the following:
\begin{align}
    \E[\sum_{q = 1}^N p_q g^q(w, \xi_t^q)] &= \sum_{q = 1}^N p_q \nabla F^q (w) & \forall w 
    \in \R^d, \forall t \in \N,
\end{align}
and
%\noindent
% Define
% \begin{equation*}
%     \E[F(w_{t + 1})] = \E_{0}\E_{1}...\E_{t}[F(w_{t + 1})].
% \end{equation*}
% then
\begin{align}
\E[\| \sum_{q=1}^{N} p_q g^q(w, \xi_t^q)\|^2_2] &\le M   & \forall w \in \R^d, \forall t \in \N,
\end{align}
where $\E[\cdot]$ is taken with respect to the joint distribution of $\{ \xi_t^1, \xi_t^2, ..., \xi_t^N \}$.
\end{assumption}

\begin{assumption}
\label{gap_assumption}
 Given a sequence of iterates $\{ w_t \}$, there exists $\rho > 0$ such that
\normalfont
\begin{equation}
\begin{aligned}
\| \text{Top}_K(\delta_{t-1} + \sum_{q = 1}^N p_q a_t^q) -& \text{Top}_K(\delta_{t-1} + \sum_{q=1}^N \text{Top}_K(p_q a_t^q)) \|_2  \nonumber \\
{}\le{}& \rho \| \alpha_{t-1} \sum_{q = 1}^N p_q g^q(w_{t-1}, \xi_{t-1}^q) \|_2 & \forall t \in \N.
\end{aligned}
\end{equation}

\end{assumption}

Assumption \ref{gap_assumption} is similar to the assumption made by Alistarh \emph{et al} \cite{alistarh2018convergence}, which is described in Assumption \ref{gap_assumption_unidirectional}.  The authors in \cite{alistarh2018convergence} explain this as a bound on the variance of the local gradients with respect to the global variance. We discuss this assumption in more detail in our toy example in supplemental material \ref{appendex:toy_example}, and compare it to Assumption \ref{gap_assumption_unidirectional} in our discussion in Section \ref{subsection:convergence_bound}.

\subsection{Convergence Bound}
\label{subsection:convergence_bound}
We present the main theorem for this section.

\begin{theorem}
\label{bidirectional_convergence}
Under Assumptions \ref{lower_bound_assumption}--\ref{gap_assumption}, if a learning rate sequence $\{ \alpha_t \}_{t \ge 0}$ and constant $\lambda \in (0, \frac{1 - \gamma}{\gamma})$ are chosen such that there exists constant $D > 0$ with
\begin{equation}
\label{convergence_condition}
    \frac{1}{1 - \gamma}\sum_{i=1}^t ((1 + \lambda)(1 - \gamma))^i  \frac{\alpha_{t-i}^2}{\alpha_t} \le D,\quad\forall t\ge 1,
\end{equation}
then running Algorithm  \ref{bidirectional_algorithm_client}  for $T$ iterations will give
\begin{equation}
\label{convergence_main}
\begin{split}
 \frac{1}{\sum_{t = 0}^T \alpha_t}\sum_{t = 0}^T \alpha_t \E[\| \nabla F(w_t)\|_2^2] & {}\le{}  \frac{2}{\sum_{t = 0}^T \alpha_t} (F(w_0) - F^*) \\
 &+  \left(L M + \frac{L^2 M D (\sqrt{1 - \gamma} +  \rho)^2}{\lambda}\right) \frac{\sum_{t = 0}^T \alpha_t^2}{\sum_{t = 0}^T \alpha_t}.
\end{split}
\end{equation}
\end{theorem}

In order for the upper bound in (\ref{convergence_main}) to converge to zero we need $\sum_{t = 1}^{\infty} \alpha_t = \infty$ and $\sum_{t = 1}^{\infty} \alpha_t^2 < \infty$. We can choose the stepsize sequence $\alpha_t = \frac{1}{(t +1)^\theta}$, $\frac{1}{2} < \theta \le 1$. Finally, we need to check if the sequence can satisfy (\ref{convergence_condition}).

We can ensure (\ref{convergence_condition}) is bounded by requiring $(1 + \lambda)(1 - \gamma) < 1$. We can always find a $\lambda$ that satisfies the inequality, since $\gamma \in (0, 1)$ from Assumption \ref{biased_compressor_assumption}. It is then easy to see that (\ref{convergence_condition}) is satisfied, since the exponential term dominates the polynomial term. 

\textbf{Remark}. Changing Assumption \ref{gap_assumption} to Assumption \ref{gap_assumption_unidirectional} gives unidirectional top$_K$ SGD the same convergence bound as bidirectional top$_K$ SGD, but with $\rho$ replaced by $\hat \rho$.

\begin{assumption}
\label{gap_assumption_unidirectional}
Given a sequence of iterates $\{ w_t \}$, there exists $\rho > 0$ such that
\normalfont
\begin{align}
\| \text{Top}_K(\sum_{q = 1}^N p_q a_t^q) - \sum_{q=1}^N \text{Top}_K(p_q a_t^q)) \|_2 &\le \hat \rho \| \alpha_{t-1} \sum_{q = 1}^N p_q g^q(w_{t-1}, \xi_{t-1}^q) \|_2 & \forall t \in \N,
 \end{align}
\end{assumption}

Bidirectional top$_K$ SGD has the same convergence bound as unidirectional top$_K$ SGD. If we choose the step size to be $\alpha_t = \frac{1}{(t + 1)^{1/2 + \epsilon}}$, where $\epsilon$ is an arbitrarily small number, the bound in Theorem \ref{bidirectional_convergence} will approach 0 at $O(1 / \sqrt{T})$. We compare the performance of bidirectional to unidirectional top$_K$ SGD in Section \ref{section:experiments}, and estimate the values of the constants in their convergence bound. Intuitively, $\rho < \Tilde{\rho}$ if the unidirectional top$_K$ update step is noisy, since applying downlink top$_K$ compression can potentially dampen noise. We provide a toy example to demonstrate this in supplemental material \ref{appendex:toy_example}.

\subsection{Proof}
\label{subsection:proof}
Similar to other error-compensated SGD analysis \cite{alistarh2018convergence, karimireddy2019error}, we consider an error-corrected sequence
\begin{equation}
\label{equation:error_corrected_sequence}
\Tilde{w}_t = w_t - \sum_{q = 1}^N p_q \epsilon_t^q - \delta_t,
\end{equation}
where $\Tilde{w}_t$ is the model parameter vector after accounting for the error term stored on all workers and server. To get the non-convex convergence bound, we apply the standard proof of SGD to $\Tilde{w}_t$, and show that $\Tilde{w}_t \approx w_t$ and $\nabla F(\Tilde{w}_t) \approx \nabla F(w_t)$. We use the following lemmas to construct the convergence bound.

\begin{lemma}
\label{equation:error_corrected_sequence_recursive_relation}
Let $\{ \Tilde{w}_t \}_{t \ge 0}$ be defined in \eqref{equation:error_corrected_sequence}, and $\{ w_t \}_{t \ge 0}$, $\{ \alpha_t \}_{t \ge 0}$, and $p_i$ for $1 \le i \le N$ be defined in Algorithm \ref{bidirectional_algorithm_client}. Then
\begin{equation}
\begin{split}
    \Tilde{w}_{t} &= \Tilde{w}_{t-1} - \alpha_{t - 1} \sum_{q = 1}^N p_q g^q(w_{t-1}, \xi^q_{t-1}).
\end{split}
\end{equation}
\end{lemma}

\begin{proof}
We have
\begin{equation}
\begin{split}
    \Tilde{w}_{t} &= w_{t} - \delta_{t} - \sum_{q = 1}^N p_q \epsilon_{t}^q \\
        &= w_{t - 1} - \delta_{t - 1} - \sum_{q=1}^N p_q a_t^q \\
        &= \Tilde{w}_{t- 1} + \sum_{q = 1}^N p_q \epsilon^q_{t-1} + \delta_{t - 1} - \delta_{t - 1} - \sum_{q=1}^N p_q a_t^q\\
        & = \Tilde{w}_{t- 1} - \alpha_{t - 1} \sum_{q = 1}^N p_q g^q(w_{t - 1}, \xi^q_{t-1}).
\end{split}
\end{equation}
\end{proof}
We use the previous lemma and Assumptions \ref{biased_compressor_assumption} and \ref{gap_assumption} to bound the difference between the error-corrected sequence $\{ \Tilde{w}_t \}$ and $\{ w_t \}$.
\begin{lemma}
\label{lemma:error_corrected_sequence_gap}
Let $\{ w_t \}_{t \ge 0}$ be defined by Algorithm $\ref{bidirectional_algorithm_client}$, and $\{ \Tilde{w}_t \}_{t \ge 0}$ be defined by (\ref{equation:error_corrected_sequence}). Under Assumptions \ref{biased_compressor_assumption} and \ref{gap_assumption}, we have
\begin{equation}
\begin{split}
    \| w_{t} - \Tilde{w}_{t}\|_2 &\le \frac{1}{\lambda}(\sqrt{1 - \gamma} + \rho)^2 \frac{1}{1 - \gamma}\sum_{i=1}^t ((1 + \lambda)(1 - \gamma))^i  \| \Tilde{w}_{t - i + 1} - \Tilde{w}_{t - i} \|^2_2,
\end{split}
\end{equation}
where one can choose the constant $\lambda \in (0, \frac{\gamma}{1 - \gamma})$.
\end{lemma}

\begin{proof} 

Applying Lemma \ref{equation:error_corrected_sequence_recursive_relation} and the iterative relation of sequence $\{ w_t \}_{t \ge 0}$ defined in Algorithms \ref{bidirectional_algorithm_client}, we get: 
\begin{align}
\begin{split}
    \| w_{t} - \Tilde{w}_{t}\|_2 {}={} &  \| w_{t-1} - \Tilde{w}_{t-1} + \alpha_{t - 1} \sum_{q = 1}^N p_q g^q(w_{t-1}, \xi_{t-1}^q) - \\
    &\text{Top}_K(\sum_{q=1}^N p_q \text{Top}_K(a_{t}^q) + \delta_{t-1}) \|_2 \\
    {}={} & \| \delta_{t - 1} + \sum_{q = 1}^N p_q \epsilon_{t-1}^q  + \alpha_{t - 1} \sum_{q = 1}^N p_q g^q(w_{t-1}, \xi_{t-1}^q) - \\
    & \text{Top}_K(\sum_{q=1}^N p_q \text{Top}_K(a_{t}^q)  + \delta_{t-1}) \|_2 \\
    {}\le{} &  \| \delta_{t - 1} + \sum_{q = 1}^N p_q a_{t}^q - \text{Top}_K(\delta_{t - 1} + \sum_{q = 1}^N p_q a_{t}^q)\|_2  + \\
    & \| \text{Top}_K(\delta_{t - 1} + \sum_{q = 1}^N p_q a_{t}^q) -\text{Top}_K(\delta_{t - 1} + \sum_{q=1}^N \text{Top}_K(p_q a_t^q)) \|_2.  \\
\end{split}
\end{align}
Using Assumption \ref{biased_compressor_assumption} and \ref{gap_assumption}, we get
\begin{align}
\label{equation:gap_without_square}
\begin{split}
    \| w_{t} - \Tilde{w}_{t}\|_2 {}\le{} & \sqrt{1 - \gamma} \|\delta_{t - 1} + \sum_{q = 1}^N \ p_q a_t^q \|_2 + \rho \| \alpha_{t-1} \sum_{q=1}^N p_q g(w_{t - 1}, \xi_{t - 1}^q) \|_2  \\ 
    {}\le{} & \sqrt{1 - \gamma} \|\delta_{t - 1} + \sum_{q=1}^N p_q \epsilon_{t - 1}^q \|_2 + \\
    &(\sqrt{1 - \gamma} + \rho) \| \alpha_{t - 1} \sum_{q=1}^N p_q g(w_{t - 1}, \xi_{t - 1}^q) \|_2 \\
    {}\le{} & \sqrt{1 - \gamma} \| w_{t - 1} - \Tilde{w}_{t - 1} \|_2 + (\sqrt{1 - \gamma} + \rho) \| \Tilde{w}_{t} - \Tilde{w}_{t - 1} \|_2,
\end{split}
\end{align}
where the final inequality in (\ref{equation:gap_without_square}) is from Lemma \ref{equation:error_corrected_sequence_recursive_relation} and the definition of the error-corrected sequence $\Tilde{w}_t$ given in (\ref{equation:error_corrected_sequence}). Taking the square, we get
\begin{equation}
\label{equation:gap}
\begin{split} 
    \|  w_{t} - \Tilde{w}_{t} \|^2_2 {}\le{}& (\sqrt{1 - \gamma} \| w_{t - 1} - \Tilde{w}_{t - 1} \|_2 + (\sqrt{1 - \gamma} + \rho) \| \Tilde{w}_{t} - \Tilde{w}_{t - 1} \|_2)^2 \\
    {}\le{}& (1 + \lambda)(1 - \gamma) \| w_{t - 1} - \Tilde{w}_{t - 1} \|^2_2 + \\
    & (1 + \frac1\lambda) (\sqrt{1 - \gamma} + \rho)^2 \| \Tilde{w}_{t} - \Tilde{w}_{t - 1} \|^2_2, \\
\end{split}
\end{equation}
which holds for any $\lambda > 0$. For reasons that will become clear later, we choose $\lambda \in (0, \frac{\gamma}{1 - \gamma})$. Iterating downwards on (\ref{equation:gap}) gives
\begin{equation}
\begin{split}
    \| \Tilde{w}_t - w_t \|^2_2 &\le \sum_{i=1}^t ((1 + \lambda)(1 - \gamma))^{i-1} (1 + \frac{1}{\lambda}) (\sqrt{1 - \gamma} + \rho)^2 \| \Tilde{w}_{t - i + 1} - \Tilde{w}_{t - i} \|^2_2 \\
    &= \frac{1}{\lambda}(\sqrt{1 - \gamma} + \rho)^2 \frac{1}{1 - \gamma}\sum_{i=1}^t ((1 + \lambda) (1 - \gamma))^i  \| \Tilde{w}
    _{t - i + 1} - \Tilde{w}_{t - i} \|^2_2 \\
\end{split}
\end{equation}
\end{proof}

Using Assumption \ref{first_and_second_moment_assumption}, we can bound the expectation of the difference between the error-corrected sequence and the true sequence from Lemma \ref{lemma:error_corrected_sequence_gap} with a constant.

\begin{lemma}
\label{lemma:expectation_bound}
Under the same setting as Lemma \ref{lemma:error_corrected_sequence_gap} as well as Assumption \ref{first_and_second_moment_assumption}, assume that 
\begin{equation*}
    \frac{1}{1-\gamma}\sum_{i=1}^t ((1 + \lambda)(1 - \gamma))^i  \frac{\alpha_{t-i}^2}{\alpha_t} \le D,
\end{equation*}
for some constant $\lambda > 0$ and $D > 0$. Then
\begin{equation*}
\begin{split}
    \E [ \| w_{t} - \Tilde{w}_{t}\|_2^2 ] &\le \frac{M}{\lambda} (\sqrt{1 - \gamma} + \rho)^2 \alpha_t D.
\end{split}
\end{equation*}
\end{lemma}

We can now prove Theorem \ref{bidirectional_convergence}.

\begin{proof}

Denote $\E_{\xi_t}[\cdot]$ to be the expectation taken with respect to the joint distribution of $\{ \xi_t^1, \xi_t^2 ..., \xi_t^N \}$ given all random variables before time $t$. Starting with Assumption \ref{L_smooth_assumption} and taking $\E_{\xi_t}[\cdot]$ on both sides of the inequality, we get
\begin{equation}
\begin{split}
    \E_{\xi_t}[F(\Tilde{w}_{t + 1})] {}\le{}& F(\Tilde{w}_{t}) + \inner{\nabla F(\Tilde{w}_{t})}{\E_{\xi_t}[\Tilde{w}_{t + 1} - \Tilde{w}_{t}]} + \frac{L}{2} \E_{\xi_t}[\| \Tilde{w}_{t + 1} - \Tilde{w}_{t}\|_2^2] \\
    {}\le{}& F(\Tilde{w}_t) - \alpha_t \inner{\nabla F(\Tilde{w}_t)}{\nabla F(w_t)} + \frac{L {\alpha_t}^2 M}{2}\\
    {}\le {}&F(\Tilde{w}_t) - \alpha_t \inner{\nabla F(w_t)}{\nabla F(w_t)} +  \frac{L {\alpha_t}^2 M}{2} + \\
    & \alpha_t \inner{\nabla F(w_t) - \nabla F(\Tilde{w}_t)}{\nabla F(w_t)} \\
    {}\le{}& F(\Tilde{w}_t) - \alpha_t \| \nabla F(w_t) \|_2^2 +  \frac{L {\alpha_t}^2 M}{2}   + \\
    & \frac{\alpha_t}{2} \| \nabla F(w_t) \|_2^2
     + \frac{\alpha_t}{2} \| \nabla F(w_t) - \nabla F(\Tilde{w}_t) \|_2^2 \\
    {}\le{}& F(\Tilde{w}_t) -  \frac{\alpha_t}{2}\| \nabla F(w_t) \|_2^2 +  \frac{L {\alpha_t}^2 M}{2}  + \frac{\alpha_t L^2}{2} \| w_t - \Tilde{w}_t \|_2^2. \\
\end{split}
\end{equation}
Taking expectation with respect to the joint distribution of all random variables on both sides of the inequality,and applying Lemma \ref{lemma:expectation_bound}, we get
\begin{equation*}
\begin{split}
\E[F(\Tilde{w}_{t+1})] &\le \E[F(\Tilde{w}_t)] -  \frac{\alpha_t}{2} \E[ \| \nabla F(w_t) \|_2^2] + \frac{L \alpha_t^2 M}{2} + \frac{\alpha_t^2 L^2}{2} \frac{M}{\lambda} (\sqrt{1 - \gamma} + \rho)^2 D.
\end{split}
\end{equation*}
Telescoping and rearranging, we get
\begin{equation*}
\begin{split}
 \sum_{t = 0}^T \alpha_t \E [\| \nabla F(w_t)\|_2^2] {}\le{}&  \sum_{t = 0}^T 2 (\E[F(\Tilde{w}_{t})] - \E[F(\Tilde{w}_{t+1})]) + L \alpha_t^2 M+ \\
 & \frac{\alpha_t^2 L^2 M D (\sqrt{1 - \gamma} +  \rho)^2}{\lambda},
\end{split}
\end{equation*}
or 
\begin{equation}
\begin{split}
 \frac{1}{\sum_{t = 0}^T \alpha_t}\sum_{t = 0}^T \alpha_t \E[\| \nabla F(w_t)\|_2^2] & {}\le{}  \frac{2}{\sum_{t = 0}^T \alpha_t} (F(w_0) - F^*) \\
 &+  \left(L M + \frac{L^2 M D (\sqrt{1 - \gamma} +  \rho)^2}{\lambda}\right) \frac{\sum_{t = 0}^T \alpha_t^2}{\sum_{t = 0}^T \alpha_t},
\end{split}
\end{equation}
as desired.
\end{proof}

% All tables must be centered, neat, clean and legible.  The table number and
% title always appear before the table.  See Table~\ref{sample-table}.

% Note that publication-quality tables \emph{do not contain vertical rules.} We
% strongly suggest the use of the \verb+booktabs+ package, which allows for
% typesetting high-quality, professional tables:
% \begin{center}
%   \url{https://www.ctan.org/pkg/booktabs}
% \end{center}
% This package was used to typeset Table~\ref{sample-table}.

% \begin{table}
%   \caption{Sample table title}
%   \label{sample-table}
%   \centering
%   \begin{tabular}{lll}
%     \toprule
%     \multicolumn{2}{c}{Part}                   \\
%     \cmidrule(r){1-2}
%     Name     & Description     & Size ($\mu$m) \\
%     \midrule
%     Dendrite & Input terminal  & $\sim$100     \\
%     Axon     & Output terminal & $\sim$10      \\
%     Soma     & Cell body       & up to $10^6$  \\
%     \bottomrule
%   \end{tabular}
% \end{table}

% Specifically, we use the author's method of assigning data to workers as well as their general idea behind synchronizing parameters across local models \cite{AshwinRJ}.
\section{Experiments}
\label{section:experiments}

The simulation code used to evaluate bidirectional top$_K$ SGD against unidirectional top$_K$ SGD is built on \cite{AshwinRJ}, which is distributed under the MIT license. All the code and models can be found in our github repository\footnote{https://github.com/wyxzou/Federated-Learning-PyTorch}. Experiments are run on the MNIST \cite{deng2012mnist} and the Fashion-MNIST \cite{xiao2017/online} datasets using multilayer perceptrons and convolution neural networks on 20, 50 and 100 workers, as well as on the CIFAR10 dataset \cite{krizhevsky2009learning} using the VGG19 network \cite{simonyan2014very} on 20 workers. For all models we choose $K = d - \left \lfloor{(1 - 0.001) d}\right \rfloor$, where $d$ is the number of parameters in the model. The $K$ value for the uplink and downlink compressor are the same for bidirectional top$_K$ SGD.

{\it MLP and CNN network on MNIST dataset}: We train a $[784, 64, 10]$ MLP model from \cite{AshwinRJ} and a CNN model with 2 convolution layers and 1 maxpooling layer from \cite{PyEx}. The MLP and CNN model has 50890 and 1199882 parameters respectively. The MNIST dataset contains 60000 train and 10000 $28 \times 28$ test images. We randomly split the 60000 elements of the training set into equal size sets and assign each to a worker. Each minibatch size is set to 10 for all models regardless of number of workers participating, and the models are trained for 100 epochs with SGD, unidirectional top$_K$ SGD, and bidirectional top$_K$ SGD. The learning rates for SGD, unidirectional top$_K$ SGD, and bidirectional top$_K$ SGD are tuned separately from 0.01 to 0.25 with step increase of 0.01, unlike previous bidirectional compression experiments, which use the same learning rate \cite{sattler2019robust, tang2019doublesqueeze}. We include the optimum learning rate of the models in Table \ref{tab:learning_rate}.

{\it MLP and CNN network on Fashion-MNIST dataset}: We also train a CNN network with 2 convolution layers and 2 max-pooling layer from \cite{AshwinRJ}, and a $[784, 256, 128, 64, 10]$ MLP network. The MLP network has 242762 parameters and the CNN network has 29034 parameters. Similar to the MNIST dataset, the Fashion-MNIST dataset contains 60000 train and 10000 $28 \times 28$ test images. We train the models using the same setup as our MNIST models, and include the optimum learning rate in Table \ref{tab:learning_rate}.
 
{\it VGG19 on CIFAR10 dataset}: Finally, we trained a VGG19 network from \cite{LiuKuang} with 20040522 parameters on the CIFAR10  dataset with 20 workers. The model is trained with batch size 100 for 200 epochs. We tune the learning rate of the VGG19 network trained with SGD on learning rates 0.01, 0.02, 0.05 and 0.1, and use this learning rate for our VGG network trained with unidirectional and bidirectional top$_K$ SGD.

We include the results for training loss and testing accuracy in Figures \ref{fig:accuracy_main} and \ref{fig:loss_main}. A more comprehensive set of test results for training loss and testing accuracy is in supplemental material  \ref{appendix:testing_accuracy_training_loss}. The results show that bidirectional top$_K$ SGD and unidirectional top$_K$ SGD will achieve similar test accuracy and training loss in the same number of epochs, consistent with our theoretical results that bidirectional and unidirectional top$_K$ SGD have similar convergence rate.

\begin{figure}[t]
\begin{minipage}{.30\textwidth}
    \centering
    \includegraphics[width=1\textwidth]{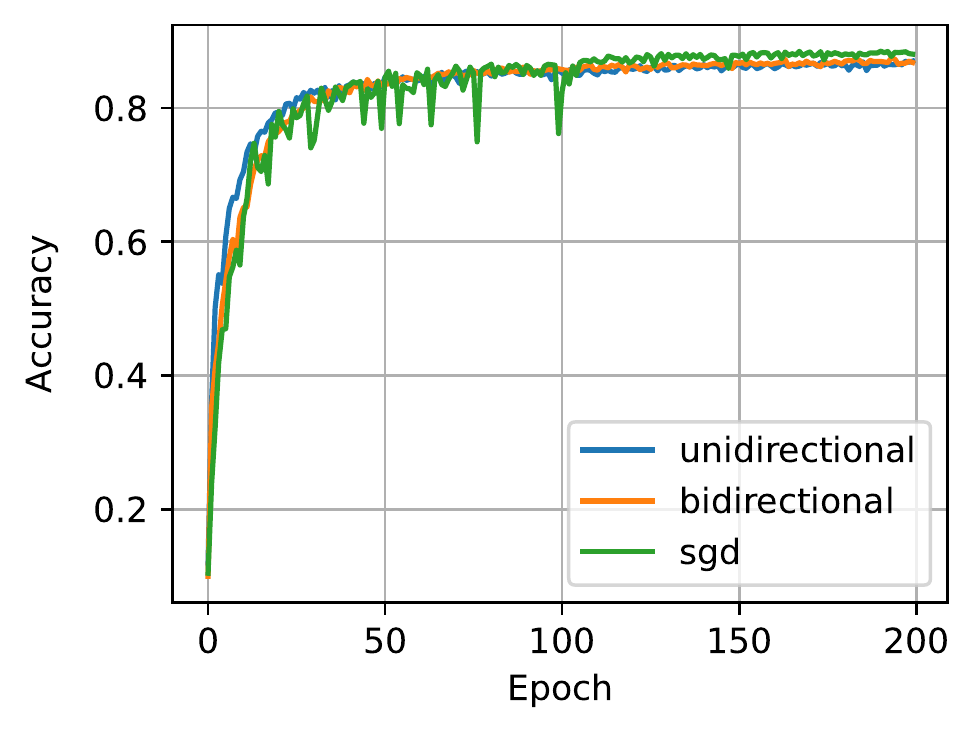}
    \caption*{VGG19 Model Trained on CIFAR10 on 20 Workers}
    \label{fig:cifar_vgg_accuracy}
\end{minipage}
\hfill
\begin{minipage}{.30\textwidth}
    \centering
    \includegraphics[width=1\textwidth]{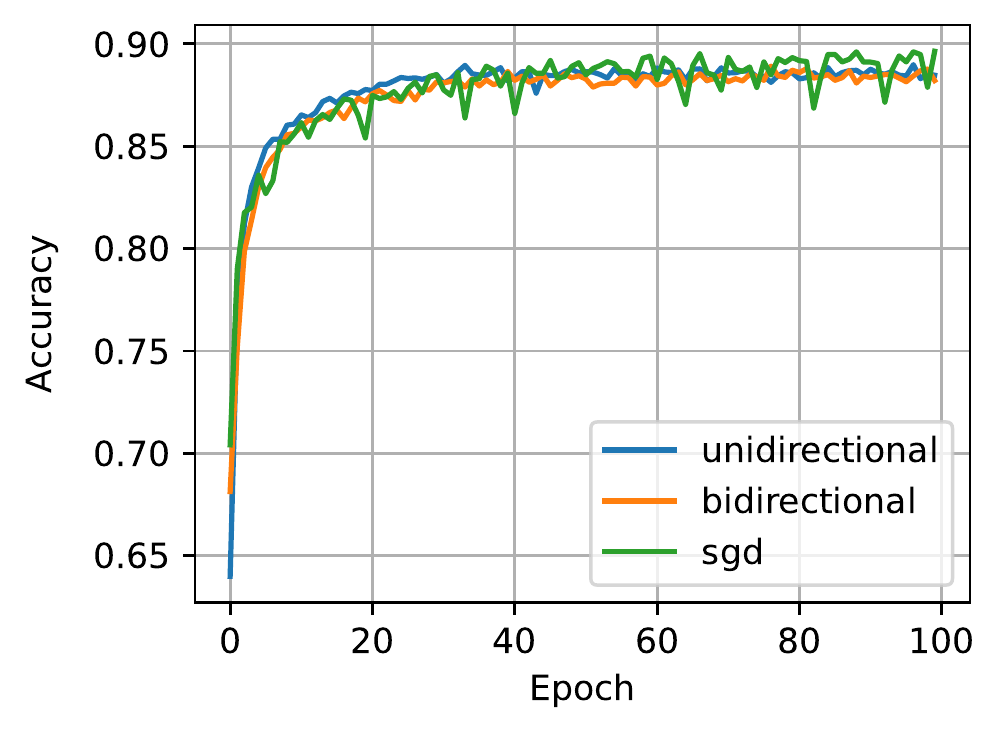}
    \caption*{MLP Model Trained on Fashion MNIST on 50 Workers}
    \label{fig:fmnist_cnn_accuracy}
\end{minipage}%
\hfill
\begin{minipage}{.30\textwidth}
    \centering
    \includegraphics[width=1\textwidth]{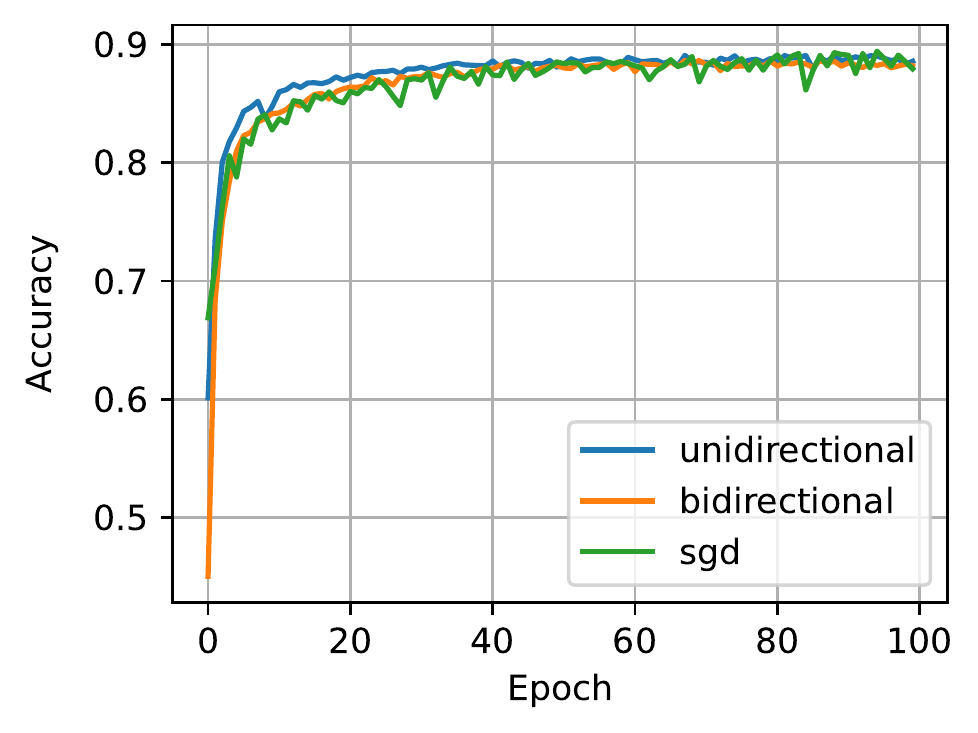}
    \caption*{MLP Model Trained on Fashion MNIST on 100 Workers}
    \label{fig:fmnist_mlp_accuracy}
\end{minipage}%
\caption{Comparison of testing accuracy.}
\label{fig:accuracy_main}
\end{figure}

\begin{figure}[t]
\begin{minipage}{.30\textwidth}
    \centering
    \includegraphics[width=1\textwidth]{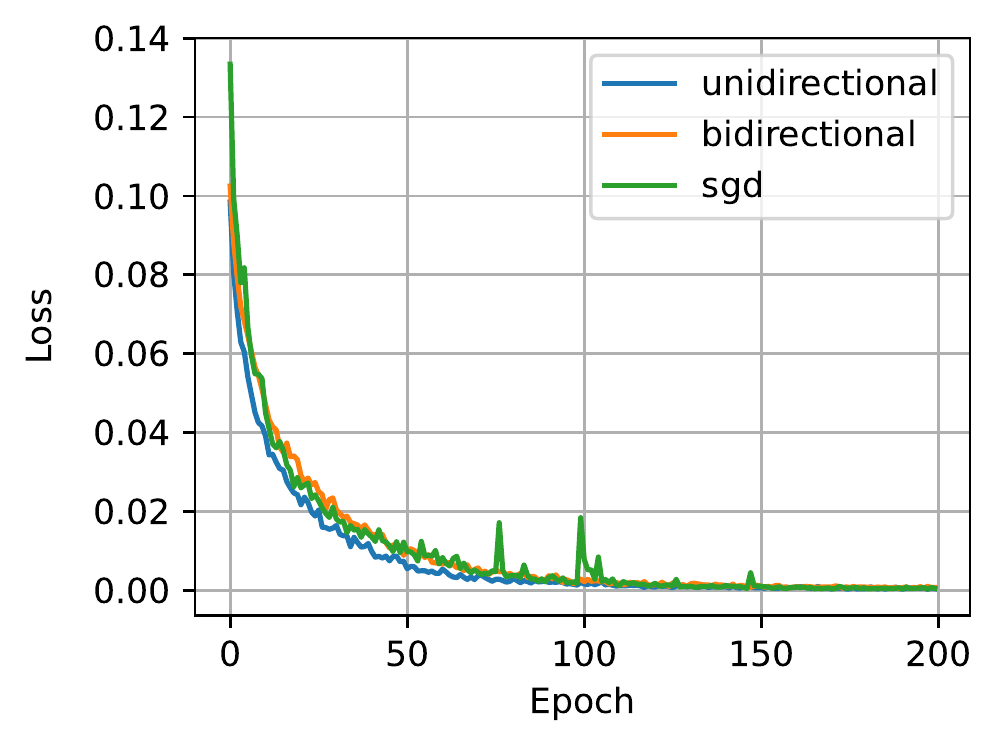}
    \caption*{VGG19 Model Trained on CIFAR10 on 20 Workers}
    \label{fig:cifar_vgg_loss}
\end{minipage}
\hfill
\begin{minipage}{.30\textwidth}
    \centering
    \includegraphics[width=1\textwidth]{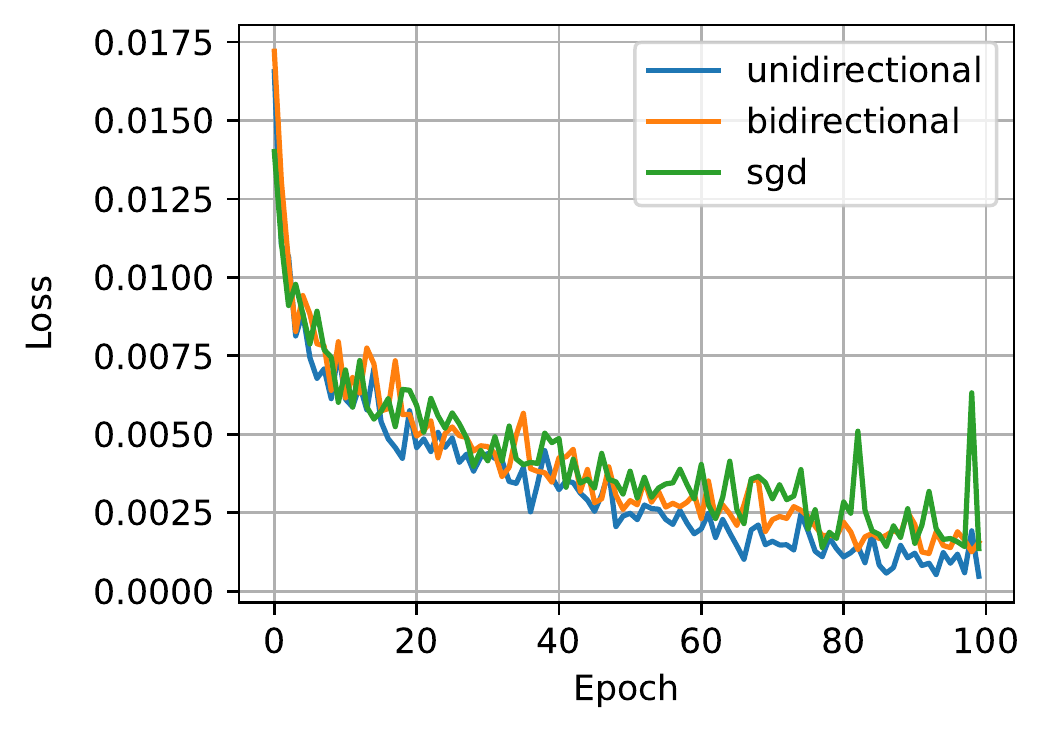}
    \caption*{MLP Model Trained on Fashion MNIST on 50 Workers}
    \label{fig:fmnist_cnn_loss}
\end{minipage}%
\hfill
\begin{minipage}{.30\textwidth}
    \centering
    \includegraphics[width=1\textwidth]{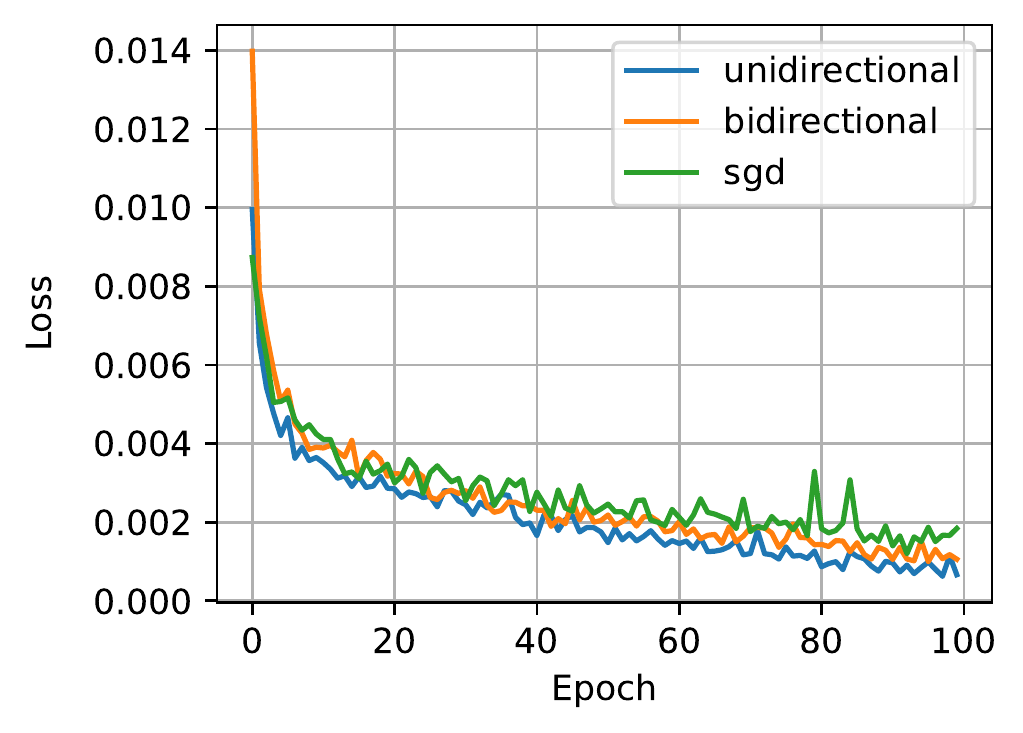}
    \caption*{CNN Model Trained on Fashion MNIST on 100 Workers}
    \label{fig:fmnist_mlp_loss}
\end{minipage}%
\caption{Comparison of training loss.}
\label{fig:loss_main}
\end{figure}

\subsection{Convergence Bound Constants}

We estimate the constants $\rho$ and $\hat \rho$ from the convergence bound of unidirectional and bidirectional top$_K$ SGD in Theorem \ref{bidirectional_convergence}, and plot them in Figure \ref{fig:rho_main}. We include plots for the rest of the tests in supplemental material \ref{appendix:rho}, and summarize the estimated $\rho$ and $\hat \rho$ values of all tests in Table \ref{tab:rho_value}. We see that the $\rho$ values for bidirectional top$_K$ SGD are consistently much smaller than the $\hat \rho$ values for unidirectional top$_K$ SGD. We also plot the maximum $1 - \gamma$ values in each epoch for bidirectional and unidirectional top$_K$ SGD in supplemental material \ref{appendix:gamma_values}, and mention that the values are approximately the same for both compression schemes. The smaller constant $\rho$ indicate that the convergence bound for bidirectional top$_K$ SGD could potentially be much smaller than the convergence bound for unidirectional top$_K$ SGD.

\begin{figure}[t]
\begin{minipage}{.30\textwidth}
    \centering
    \includegraphics[width=1\textwidth]{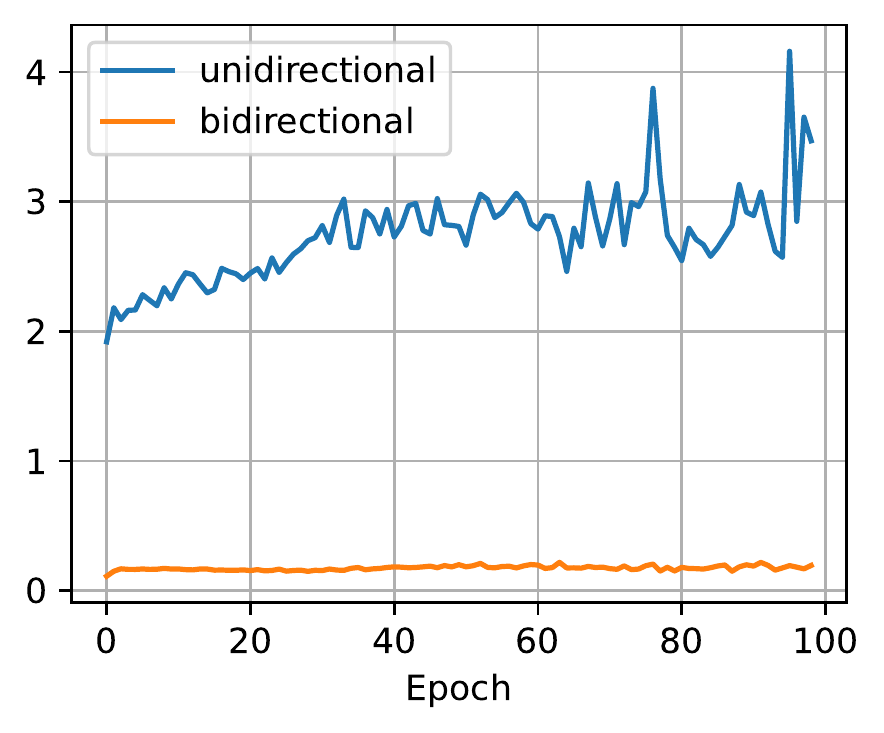}
    \caption*{VGG19 Model Trained on CIFAR10 on 20 Workers}
\end{minipage}
\hfill
\begin{minipage}{.30\textwidth}
    \centering
    \includegraphics[width=1\textwidth]{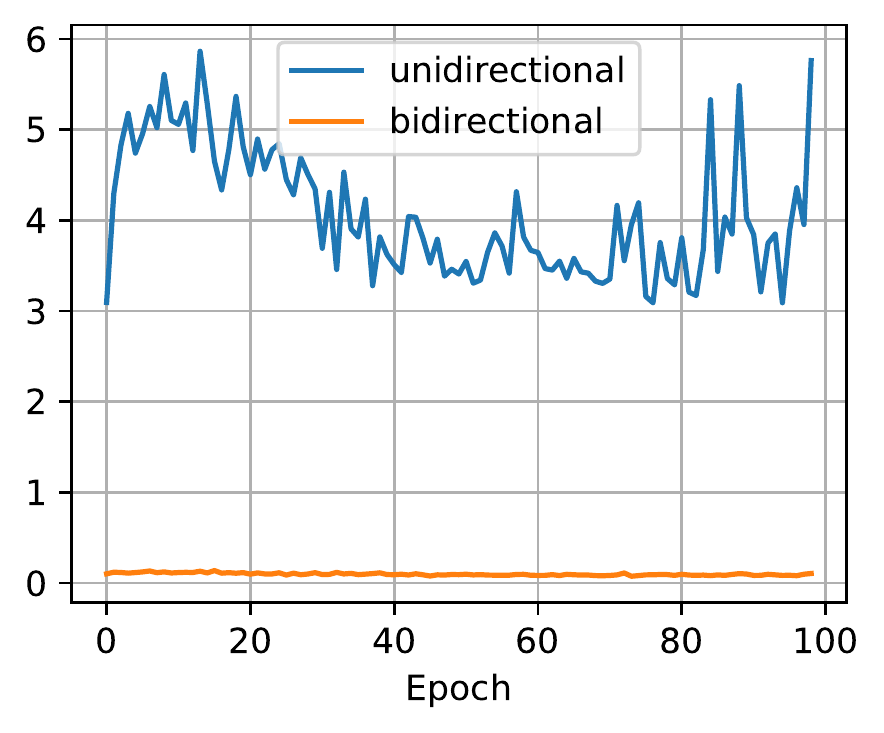}
    \caption*{MLP Model Trained on Fashion MNIST on 50 Workers}
\end{minipage}%
\hfill
\begin{minipage}{.30\textwidth}
    \centering
    \includegraphics[width=1\textwidth]{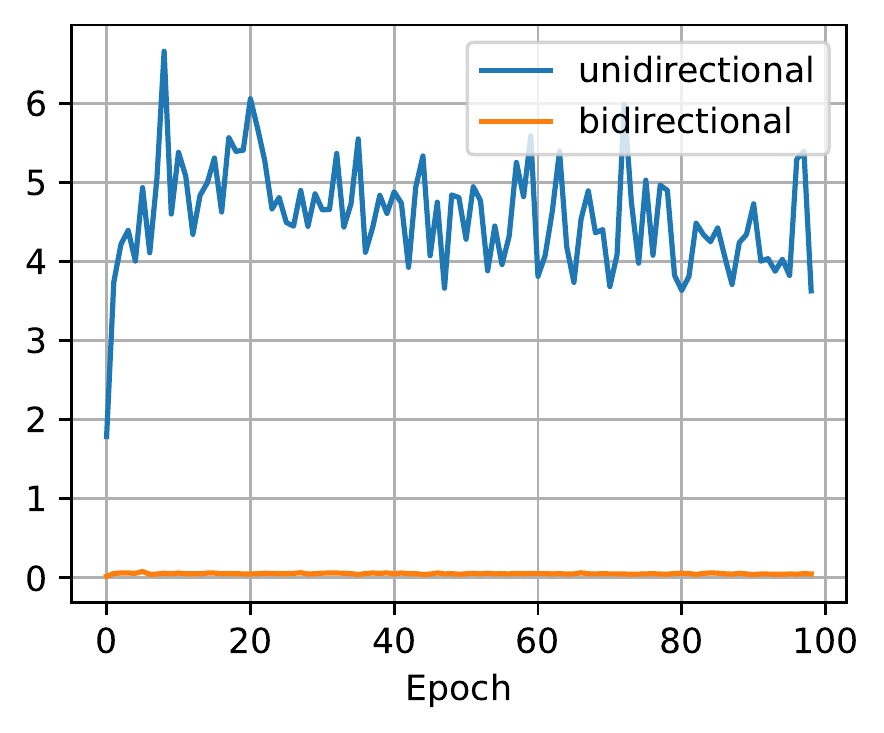}
    \caption*{MLP Model Trained on Fashion MNIST on 100 Workers}
\end{minipage}%
\caption{Largest $\rho$ and $\hat \rho$ value in each epoch.}
\label{fig:rho_main}
\end{figure}

\begin{table}[ht]
\caption{Maximum $\hat \rho$ and $\rho$ value of trained models across all epochs.}
\centering
%     \begin{adjustbox}{width=\textwidth,center}
    %n\begin{adjustbox}{center}
   .    \renewcommand{\arraystretch}{1.5}
        \begin{tabular}{c c c c c }
            \hline
            Dataset                & Model                 & Workers &    $\hat \rho$  &   $\rho$  \\ \hline
            \multirow{8}{*}{Fashion MNIST} & \multirow{4}{*}{MLP}  & 20      &    21.78     &    0.60  \\
                                   &                       & 50      &    5.86     &    0.14  \\
                                   &                       & 100     &    6.65     &    0.08  \\\cline{2-5}                                 
                                   & \multirow{4}{*}{CNN}  & 20    &  15.35        &    0.92  \\
                                   &                       & 50    &    7.48      &    0.24  \\
                                   &                       & 100   &    9.05     &    0.18  \\\cline{1-5}                                 
            \multirow{8}{*}{MNIST}& \multirow{4}{*}{MLP}  & 20      &    16.03     &    0.95 \\
                                   &                       & 50      &    9.72    &   0.28 \\
                                   &                       & 100     &   13.90    &    0.15  \\\cline{2-5}                                 
                                   & \multirow{4}{*}{CNN}  & 20    &    306.19      &    33.60  \\
                                   &                       & 50    &   24.26     &    1.09 \\
                                   &                       & 100   &    6.14      &    0.09 \\\cline{1-5}    
            CIFAR10 & VGG19 & 20 & 4.16 & 0.21  \\                                                 \hline
        \end{tabular}
    %\end{adjustbox}
%     \vspace{ - 05 mm}
   
    \label{tab:rho_value}
\end{table}

\subsection{Communication Saved}
\label{subsection:communication_saved}
 Unlike uplink compression, the number of bits saved from downlink compression is related to the number of participating workers. If $N$ workers contributed sparse gradients with $K_\text{uplink}$ non-zero indices, then the gradient sent back by the server after aggregation is a sparse gradient with at most $K_\text{uplink} N$ non-zero indices, and applying top$_K$ sparsification in the downlink will compress the gradient to at most $K_\text{downlink}/ (K_\text{uplink} N)$ of its size. We measure the percentage of non-zero indices in the sum of gradients from the workers, as shown in Figure \ref{fig:fmnist_mlp_compression} and supplemental material \ref{appendix:non-zero}. For all experiments, we see that the fraction approaches  $K_\text{uplink} N / d$ as the number of iterations increase, showing that the compression rate of the downlink top$_K$ compressor is almost as large as possible.
 
The time it takes to transfer a gradient from worker-to-server then server-to-worker for unidirectional compression is
\begin{equation}
    \alpha_1 + 2 K_\text{uplink} \beta_1 + \alpha_2 + 2 N K_\text{uplink} \beta_2,
\end{equation}
and the time it takes to transfer a message in bidirectional compression is
\begin{equation}
    \alpha_1 + 2 K_\text{uplink} \beta_1 + \alpha_2 +  2 K_\text{downlink} \beta_2,
\end{equation}
where $\alpha_1$ is the uplink latency, $\alpha_2$ is the downlink latency, $\beta_1$ is the uplink transfer time for a 32-bit float, and $\beta_2$ is the downlink transfer time for a 32-bit float.
 
 \begin{figure}[ht]
    \centering
    \includegraphics[width=1\textwidth]{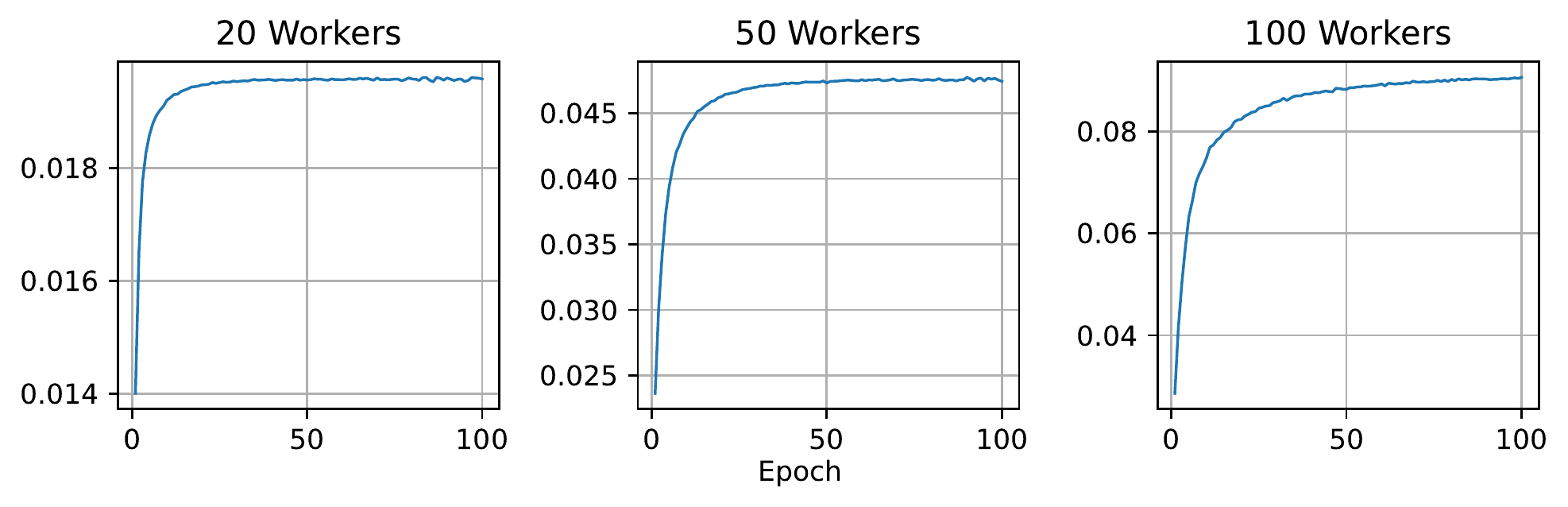}
    \caption{Fraction of non-zero indices after aggregating sparse gradients from workers. Trained on a MLP network using Fashion MNIST dataset. $K_\text{uplink} = K_\text{downlink} \approx  0.001d$.}
    \label{fig:fmnist_mlp_compression}
\end{figure}

\FloatBarrier

\section{Conclusion}

We demonstrate that bidirectional top$_K$ SGD that can potentially have a tighter convergence bound than unidirectional top$_K$ SGD while providing significant communication benefits. We remove the restriction of Alistarh \emph{et al.}'s non-convex analysis of unidirectional top$_K$ SGD that requires $K > \frac12$ \cite{alistarh2018convergence}. We provide testing across different models, datasets, and number of workers at state-of-the-art sparsification levels \cite{chen2021distributed} to show that bidirectional top$_K$ SGD can converge as well as unidirectional top$_K$ SGD. Our work shows that bidirectional compression should always be used, especially for large number of workers, because server-to-worker communication can be reduced by a factor proportional to $N$, without affecting convergence speed or accuracy.

\medskip

{
\small

% Bibliography

% The following statement selects the style to use for references.  
% It controls the sort order of the entries in the bibliography and also the formatting for the in-text labels.
\bibliographystyle{plain}
% This specifies the location of the file containing the bibliographic information.  
% It assumes you're using BibTeX to manage your references (if not, why not?).

% \bibliography{main.bib}
}

\clearpage

\appendix

\section{Appendix}

\subsection{Toy Example}
\label{appendex:toy_example}

Consider solving the distributed problem with 3 workers

\begin{equation*}
    \min_{w \in R^{100}} F(w) \triangleq \frac13 \sum_{q = 1}^3  F^q(w),
\end{equation*}
where

\begin{align*}
    &F^1(w) \triangleq \frac12 (w - \vec{1})^T (w - \vec{1}), \\
    &F^2(w) \triangleq  \frac12 (w - \vec{5})^T (w - \vec{5}), \\
    &F^3(w) \triangleq  \frac12 (w - \vec{10})^T (w - \vec{10}).
\end{align*}

The minimum value of $F(w)$ is $\frac{6100}{9}$, when $w = \vec{\frac{16}{3}}$. We initialize $w_0 \in \R^{100}$ generated from  $\mathcal{N}(20,\,1)$ with random seed $10$, and run Algorithm \ref{unidirectional_algorithm_client} and \ref{bidirectional_algorithm_client} with $K = 1$. Note that since we can solve the gradient, there is no stochastic element.

Both unidirectional (uplink) and bidirectional top$_K$ SGD  oscillates periodically at a distance away from the optimal value. However, bidirectional top$_K$ SGD converges closer to the $F^*$ than unidirectional uplink top$_K$ SGD, which is counter-intuitive, since we should be losing information by adding downlink compression.
 
 We provide an explanation of this by observing the gradients from each worker in unidirectional top$_K$ SGD at $t = 210$.  $a_t^q$ from each worker $q$ in line 3 of Algorithm \ref{unidirectional_algorithm_client} is shown in Figure \ref{fig:gradient_frozen_frame}. We note that the non-zero components from the sum of top$_K$ updates, $\sum_{i=1}^3 \text{Top}_K(a_t^q)$, shown in Figure \ref{fig:sparse_frozen_frame} is very different from the corresponding components in the sum of updates, $\sum_{i=1}^3 a_t^q$, shown in Figure \ref{fig:gradient_without_sparsification}. This happens because the components from worker 1 and worker 3 have opposite signs. If the largest gradient component from workers 1 and 3 do not have the same index, then $\sum_{i=1}^3 \text{Top}_K(a_t^q)$ will be far from $\sum_{i=1}^3 a_t^q$. In our example, the norm of the difference between the full update and unidirectional (uplink) update step is greater than the norm of the difference between the full and the bidirectional update. Specifically, for unidirectional (uplink) top$_K$ SGD, at $t = 210$,
\begin{equation*}
\begin{split}
& \| \sum_{q=1}^3 a_t^q - \sum_{q=1}^3 \text{Top}_K(a_t^q) \|_2  = 21.80, \\
\end{split}
\end{equation*}
while
\begin{equation*}
\begin{split}
& \| \sum_{q=1}^3 a_t^q - \text{Top}_K(\sum_{q=1}^3 \text{Top}_K(a_t^q)) \|_2  = 21.54, \\
\end{split}
\end{equation*}
showing that adding downlink top$_K$ sparsification brings the update closer to the uncompressed update.  While applying downlink compression causes the gradient to lose information, it also causes the gradient to lose ``bad" information. This motivates us to split the error into 2 parts for unidirectional top$_K$ SGD,
\begin{equation}
\label{unidirectional_error_split}
\begin{split}
& \| \sum_{q=1}^N p_q a_t^q  - \sum_{q=1}^N p_q \text{Top}_K(a_t^q) \|_2 \\
&\le \| \sum_{q=1}^N p_q a_t^q - \text{Top}_K(\sum_{q=1}^N p_q a_t^q) \|_2 + \| \text{Top}_K(\sum_{q=1}^N p_q a_t^q) - \sum_{q=1}^N p_q \text{Top}_K(a_t^q) \|_2.
\end{split}
\end{equation}

One benefit of the new representation is that it has nice physical meaning. The first norm is the error inherent to the compressor and can be bounded by $\gamma$-approximate compressor defined in Assumption (\ref{biased_compressor_assumption}), and the second norm is the error that comes from the distributed system, which occurs when the gradients from the local workers are not representative of the global gradient. Plotting $\| \text{Top}_K(\sum_{q=1}^N p_q a_t^q) - \sum_{q=1}^N p_q \text{Top}_K(a_t^q) \|_2$ and $\| \text{Top}_K(\delta_{t-1} + \sum_{q=1}^N p_q a_t^q) - \text{Top}_K(\delta_{t-1} + \sum_{q=1}^N p_q \text{Top}_K(a_t^q)) \|_2$ in Figure \ref{fig:ratio_20} for bidirectional and unidirectional downlink top$_K$ SGD, we see that adding a downlink top$_K$ compressor to unidirectional top$_K$ SGD can reduce the distributed error. Intuitively, the update step is noisy, and applying an extra top$_K$ compression would reduce the noise.

\setlength{\textfloatsep}{10pt plus 1.0pt minus 2.0pt}
\begin{figure}[ht]
    \centering
    \includegraphics[width=\textwidth]{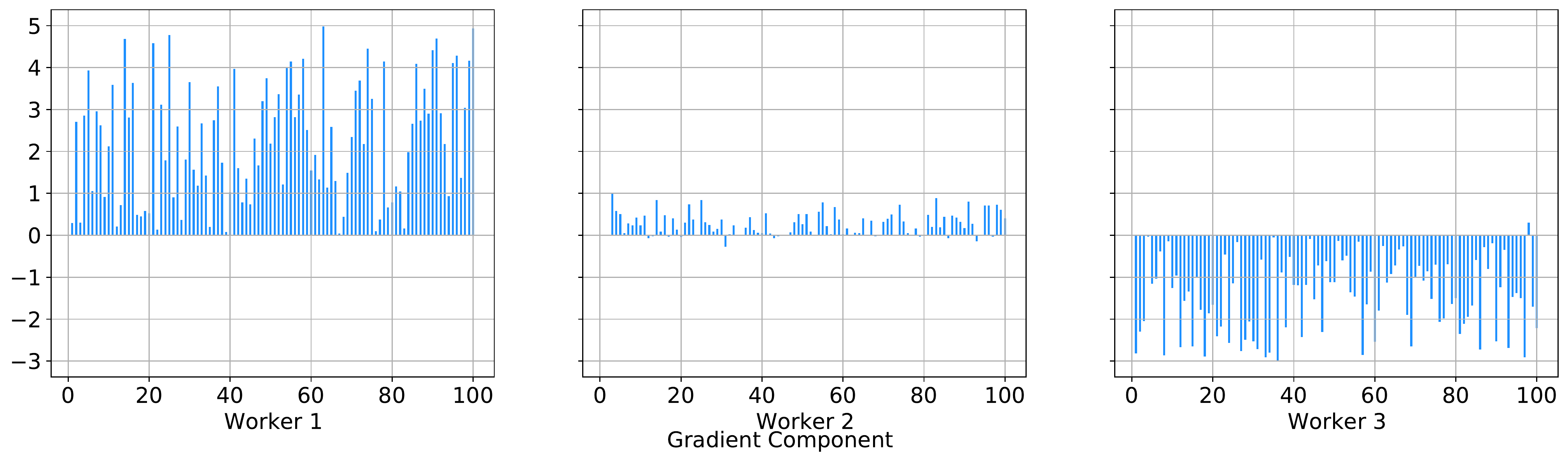}
    \caption{$a_t^q$ from each worker at $t = 300$ for unidirectional top$_K$ SGD.}
    
    \label{fig:gradient_frozen_frame}
\end{figure}

\begin{figure}[t]
\begin{minipage}{.47\textwidth}
    \centering
    \includegraphics[width=1\textwidth]{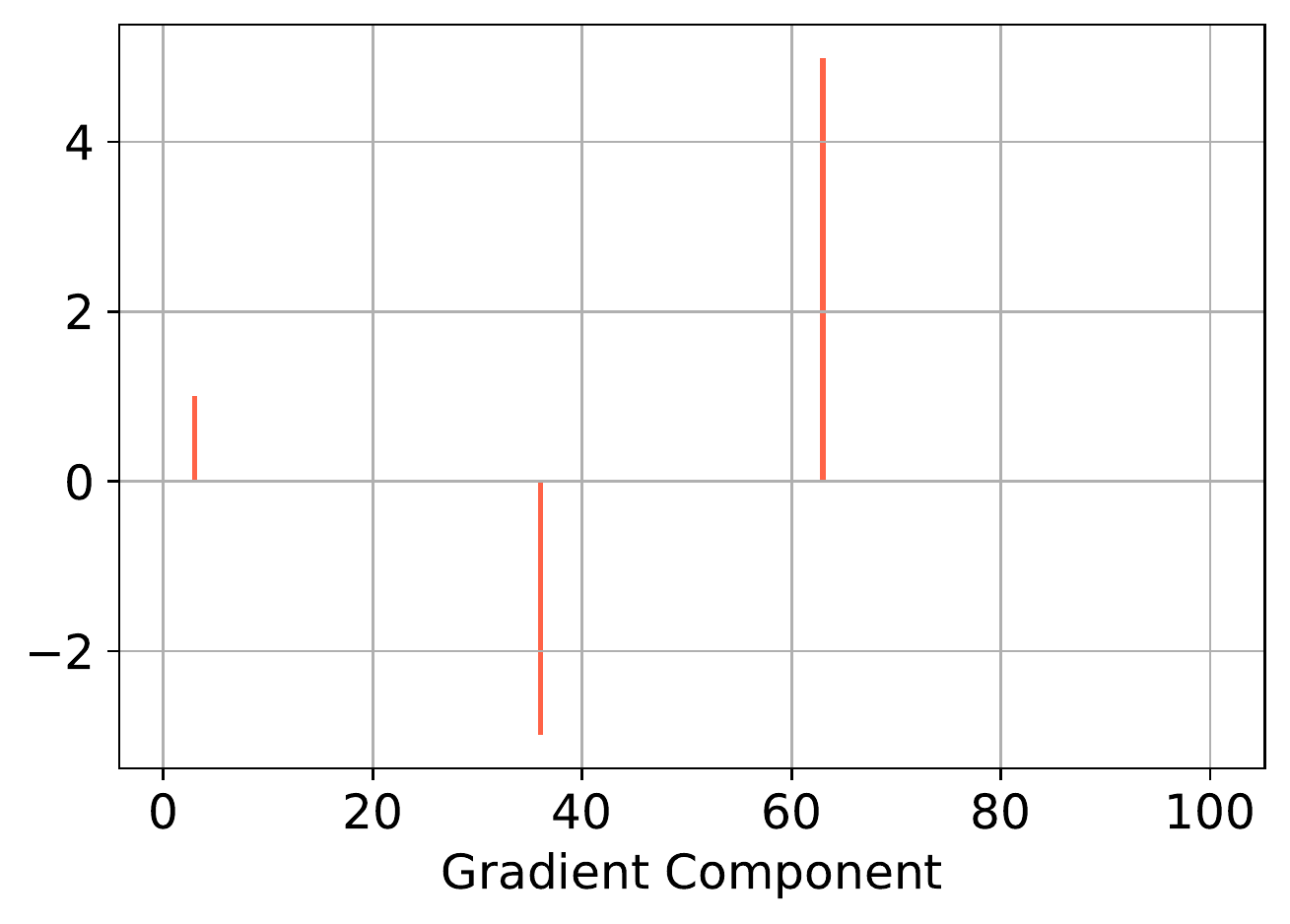}
    \caption{$\sum_{q=1}^N \text{Top}_K(a_t^q)$ at $t = 300$ for unidirectional top$_K$ SGD.}
    \label{fig:sparse_frozen_frame}
\end{minipage}
\hfill
\begin{minipage}{.47\textwidth}
    \centering
    \includegraphics[width=1\textwidth]{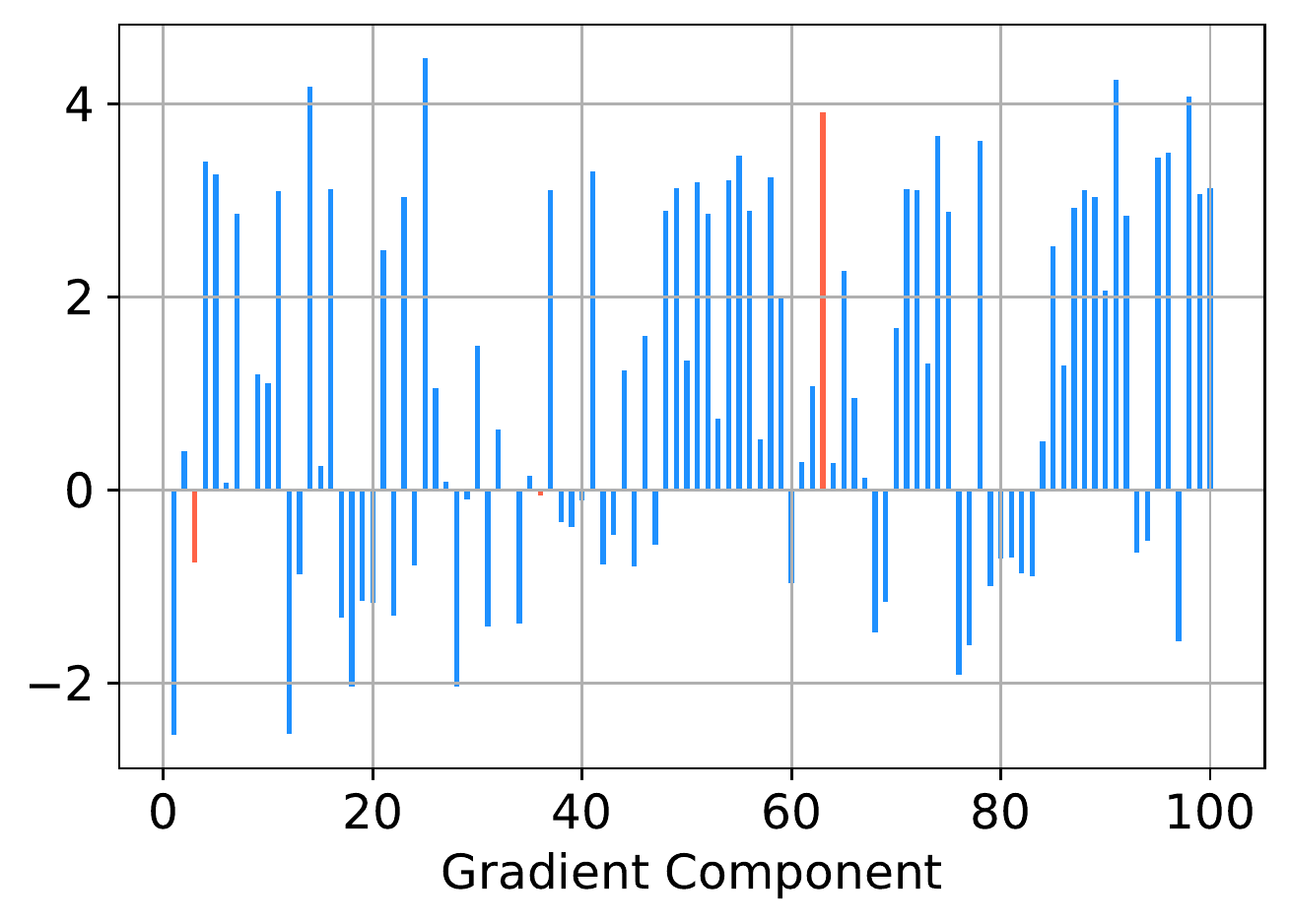}
    \caption{$\sum_{q=1}^N a_t^q$ at $t = 300$ for unidirectional top$_K$ SGD.}
    \label{fig:gradient_without_sparsification}
\end{minipage}%
\end{figure}

\begin{figure}[ht]
\begin{minipage}{.47\textwidth}
    \centering
    \includegraphics[width=1\textwidth]{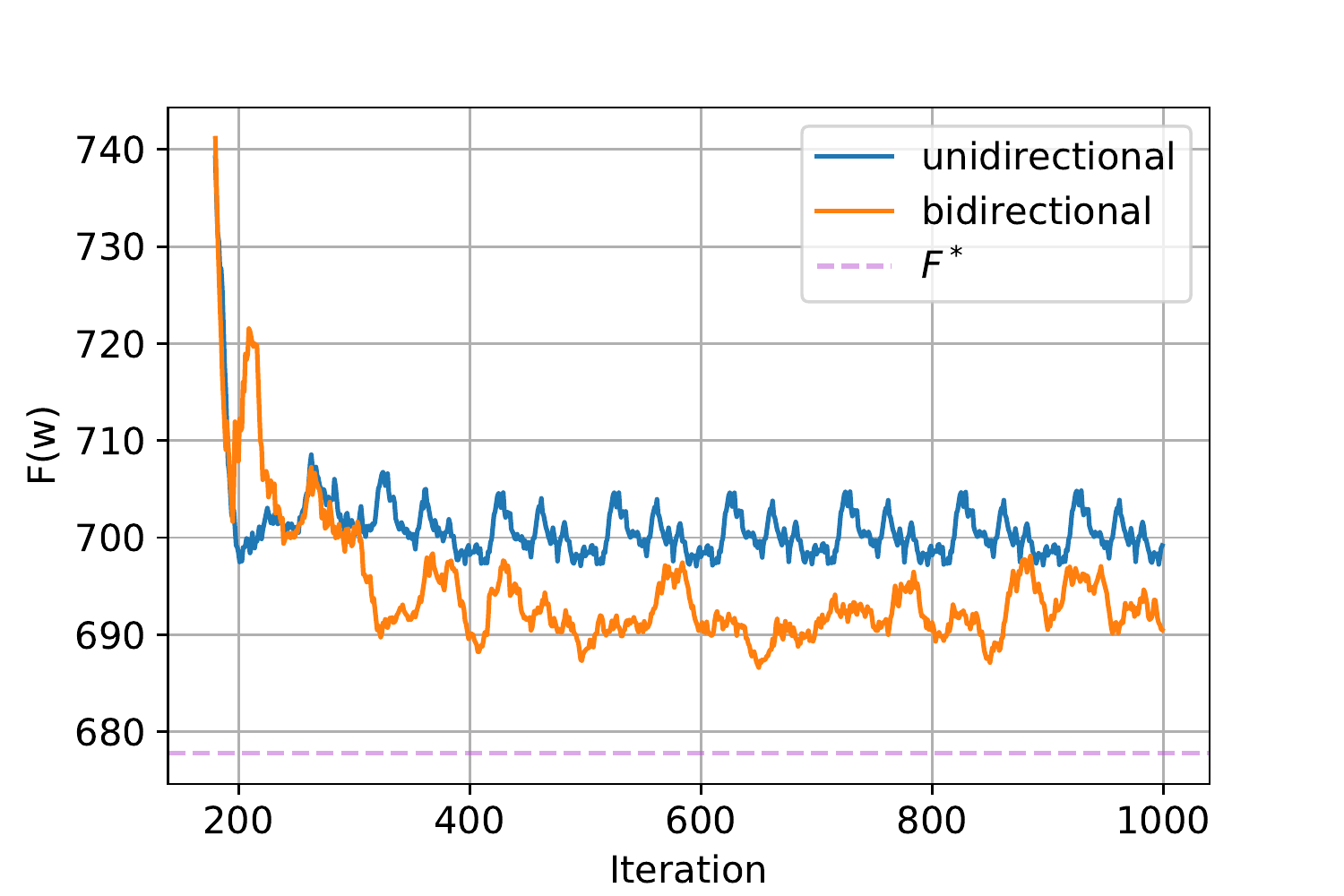}
    \caption{Convergence of unidirectional vs bidirectional top$_K$ SGD run with $K=1$ and $\alpha = 0.01$.}
    \label{fig:toy_example_convergence}
\end{minipage}
\hfill
\begin{minipage}{.47\textwidth}
    \centering
    \includegraphics[width=1\textwidth]{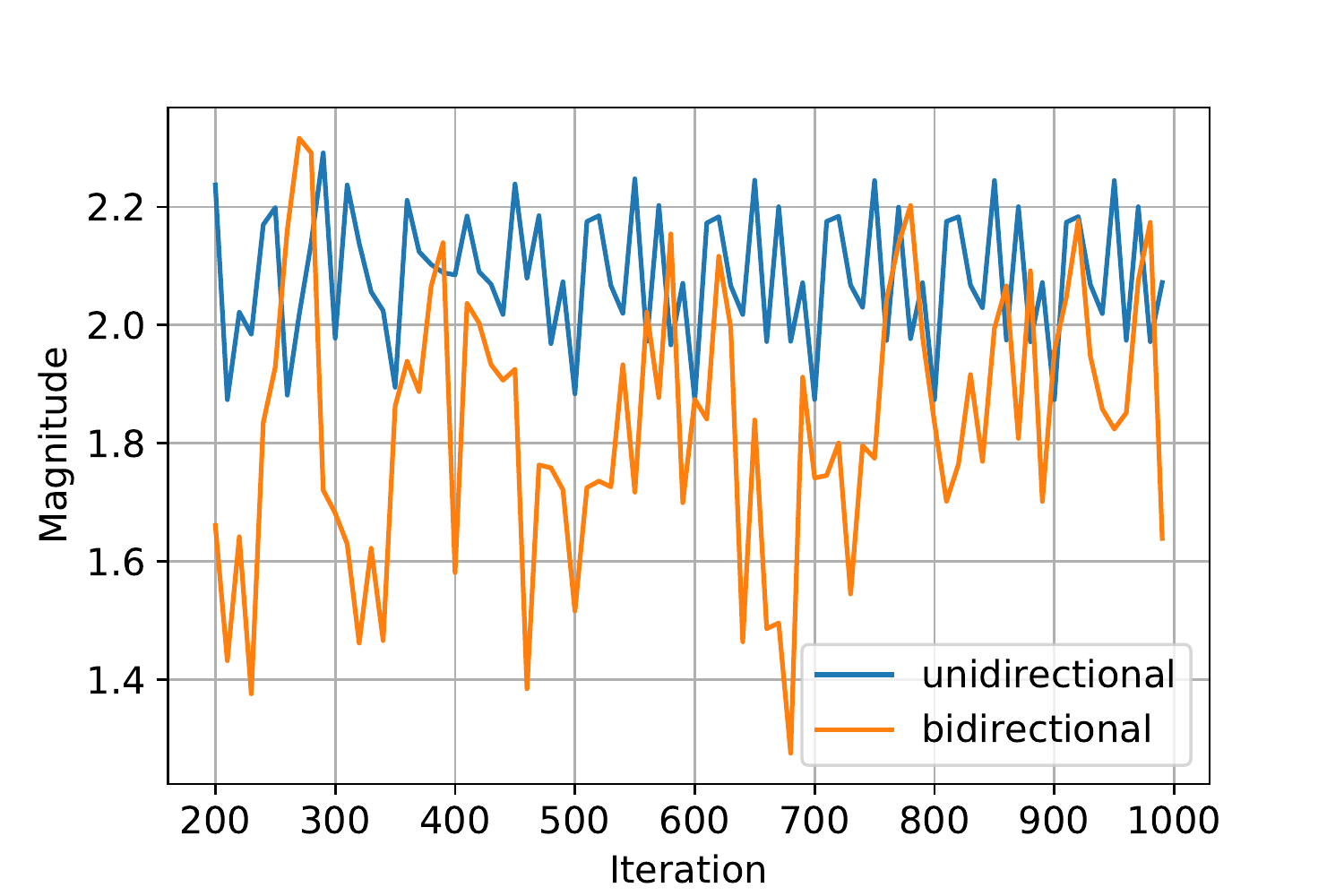}
    \caption{Average $\| \text{Top}_K(\sum_{q = 1}^N p_q a_t^q) - \sum_{q=1}^N \text{Top}_K(p_q a_t^q) \|_2$ values every 10 iterations for unidirectional top$_K$ SGD and average $\| \text{Top}_K(\delta_{t-1} + \sum_{q = 1}^N p_q a_t^q) - \text{Top}_K(\delta_{t-1} + \sum_{q=1}^N \text{Top}_K(p_q a_t^q)) \|_2$ values every 10 iterations for bidirectional top$_K$ SGD.}
    \label{fig:ratio_20}
\end{minipage}
\end{figure}

\FloatBarrier

\subsection{Learning Rates}
\label{appenix:learning_rates}
\begin{table}[ht]
\caption{Learning rate chosen for all models.}

\centering
%     \begin{adjustbox}{width=\textwidth,center}
    %n\begin{adjustbox}{center}
   .    \renewcommand{\arraystretch}{1}
        \begin{tabular}{c c c c c c }
            \hline
            Dataset                & Model                 & Workers &    unidirectional lr  &   bidirectional lr & SGD lr \\ \hline
            \multirow{8}{*}{Fashion MNIST} & \multirow{4}{*}{MLP}   & 20      &    0.08     &    0.08  &  0.06  \\
                                   &                       & 50      &    0.13     &    0.12  &  0.07 \\
                                   &                       & 100     &    0.22     &    0.22 & 0.08 \\\cline{2-6}                                 
                                   & \multirow{4}{*}{CNN}   & 20    &  0.09       &    0.08  &  0.06 \\
                                   &                       & 50    &    0.12      &    0.11  &  0.07 \\
                                   &                       & 100   &  0.14     &    0.2 &  0.08 \\\cline{1-6}                                 
            \multirow{8}{*}{MNIST}& \multirow{4}{*}{MLP}   & 20      &   0.06     &   0.09  & 0.03 \\
                                   &                       & 50      &    0.17    &   0.18  & 0.1 \\
                                   &                       & 100     &   0.17   &    0.24  &  0.1 \\\cline{2-6}                                 
                                   & \multirow{4}{*}{CNN}  & 20    &    0.08     &   0.09  & 0.05\\
                                   &                       & 50    &   0.14     &    0.16 & 0.07 \\
                                   &                       & 100   &    0.09      &    0.16 & 0.07 \\\cline{1-6}    
            CIFAR10 & VGG19 & 20 & 0.05 & 0.05 & 0.05 \\                                                 \hline
        \end{tabular}
    %\end{adjustbox}
%     \vspace{ - 05 mm}
    \label{tab:learning_rate}
\end{table}

\FloatBarrier

% \begin{table}
%   \caption{Sample table title}
%   \label{sample-table}
%   \centering
%   \begin{tabular}{lll}
%     \toprule
%     \multicolumn{2}{c}{Part}                   \\
%     \cmidrule(r){1-2}
%     Name     & Description     & Size ($\mu$m) \\
%     \midrule
%     Dendrite & Input terminal  & $\sim$100     \\
%     Axon     & Output terminal & $\sim$10      \\
%     Soma     & Cell body       & up to $10^6$  \\
%     \bottomrule
%   \end{tabular}
% \end{table}

\subsection{Experiment Figures}

\subsubsection{Testing Accuracy and Training Loss}
\label{appendix:testing_accuracy_training_loss}

\begin{figure}[ht]
    \centering
    \includegraphics[width=1\textwidth]{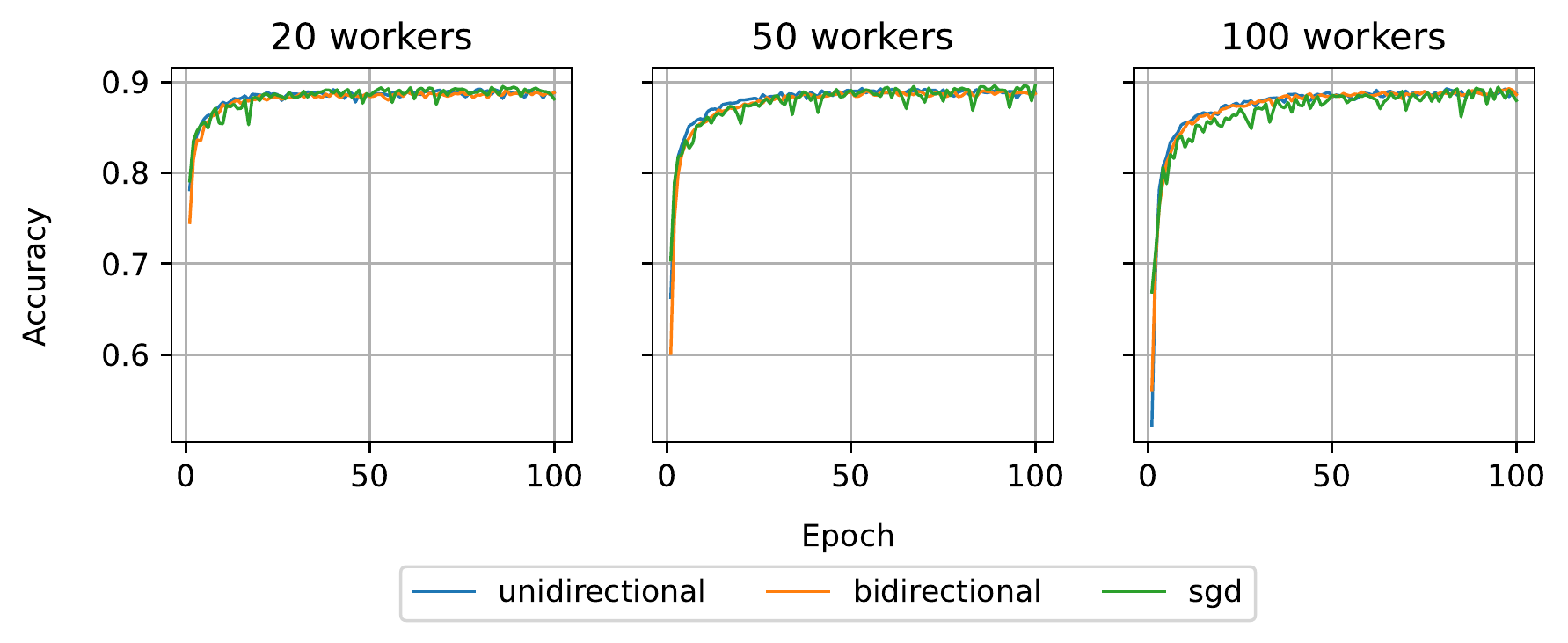}
    \caption{Comparison of testing accuracy for MLP model trained on Fashion MNIST data for unidirectional top$_{K}$,  bidirectional top$_{K}$ and vanilla SGD. Trained on batch size 10. $K_\text{downlink} = K_\text{uplink} \approx 0.001 d$ for bidirectional. $K_\text{uplink} \approx 0.001 d$ for unidirectional.}
    \label{fig:fmnist_mlp_test_accuracy}
\end{figure}

\begin{figure}[ht]
    \centering
    \includegraphics[width=1\textwidth]{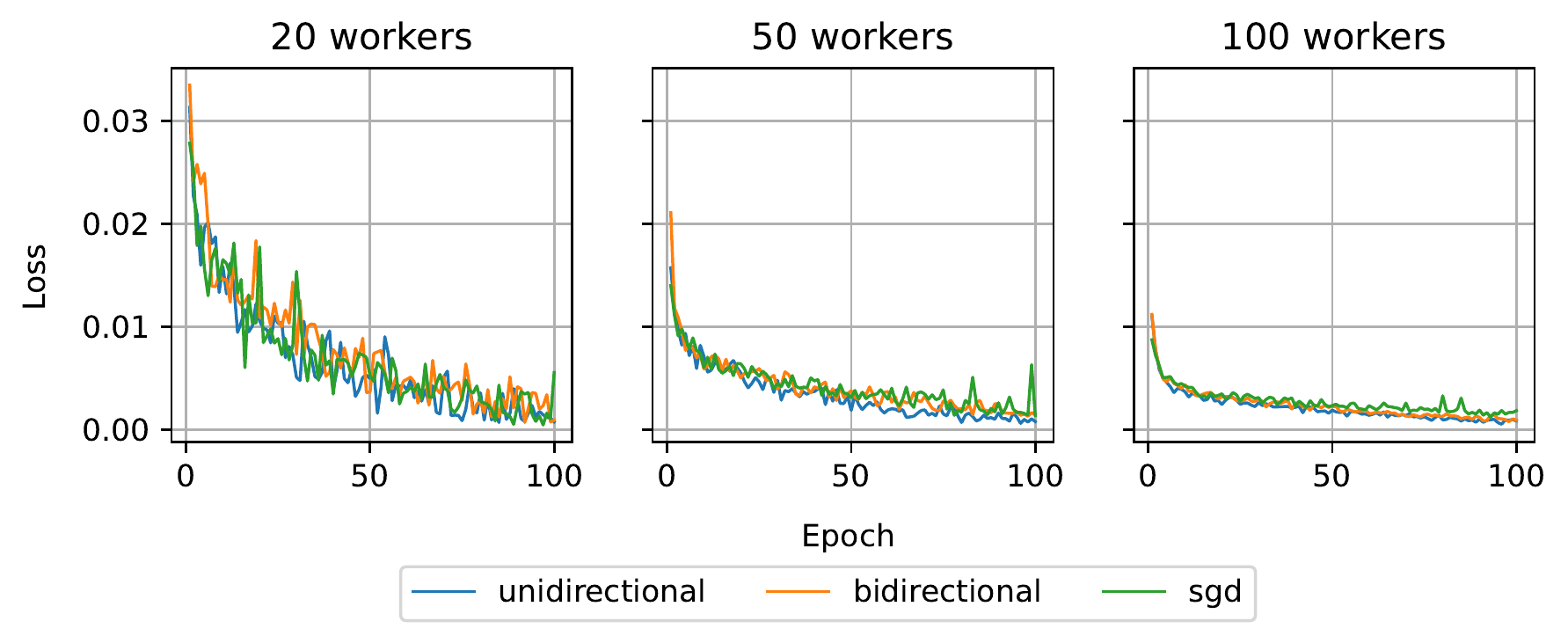}
     \caption{Comparison of training loss for MLP model trained on Fashion MNIST data for unidirectional top$_{K}$,  bidirectional top$_{K}$ and vanilla SGD. Trained on batch size 10. $K_\text{downlink} = K_\text{uplink} \approx 0.001 d$ for bidirectional. $K_\text{uplink} \approx 0.001 d$ for unidirectional.}
    \label{fig:fmnist_mlp_train_loss}
\end{figure}
% \subsection{Fashion MNIST CNN Model}
\begin{figure}[ht]
    \centering
    \includegraphics[width=1\textwidth]{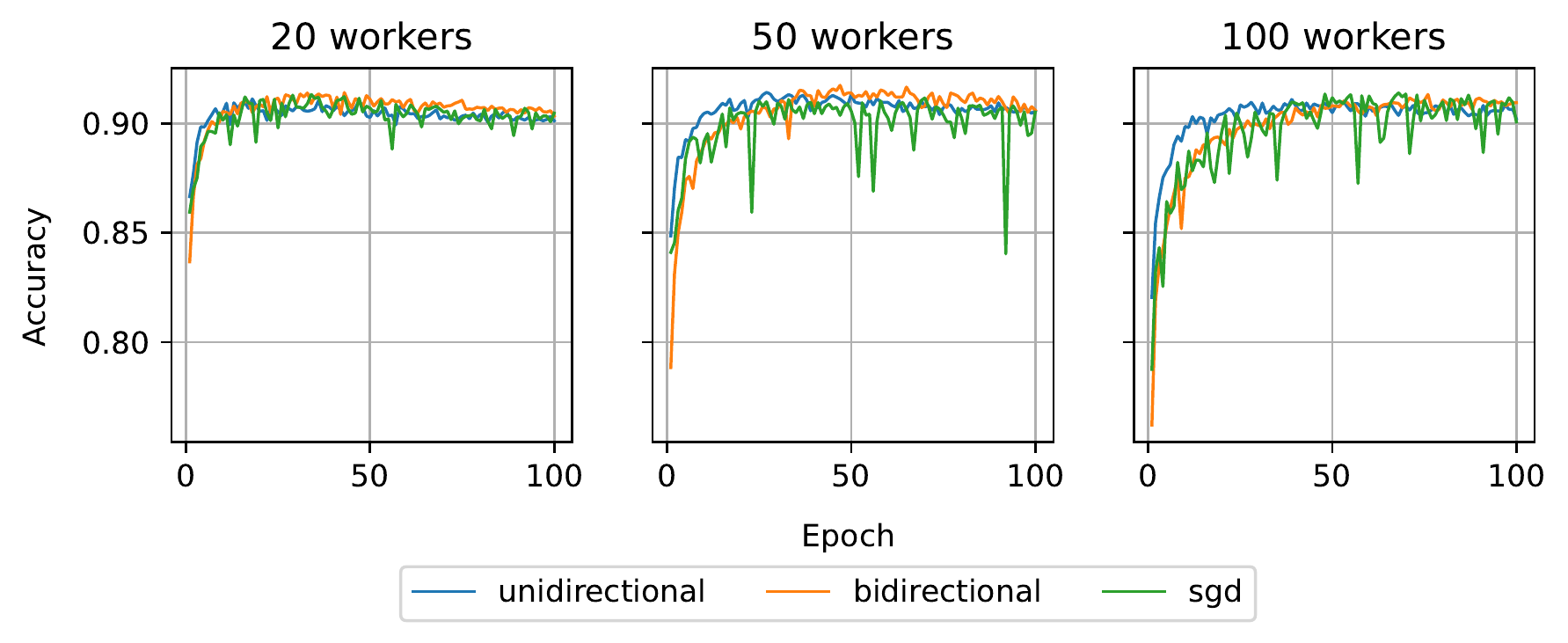}
    \caption{Comparison of testing accuracy from training CNN model on Fashion MNIST data for unidirectional top$_{K}$,  bidirectional top$_{K}$ and vanilla SGD. Trained on batch size 10. $K_\text{downlink} = K_\text{uplink} \approx 0.001 d$ for bidirectional. $K_\text{uplink} \approx 0.001 d$ for unidirectional.}
    \label{fig:fmnist_cnn_test_accuracy}
\end{figure}

\begin{figure}[ht]
    \centering
    \includegraphics[width=1\textwidth]{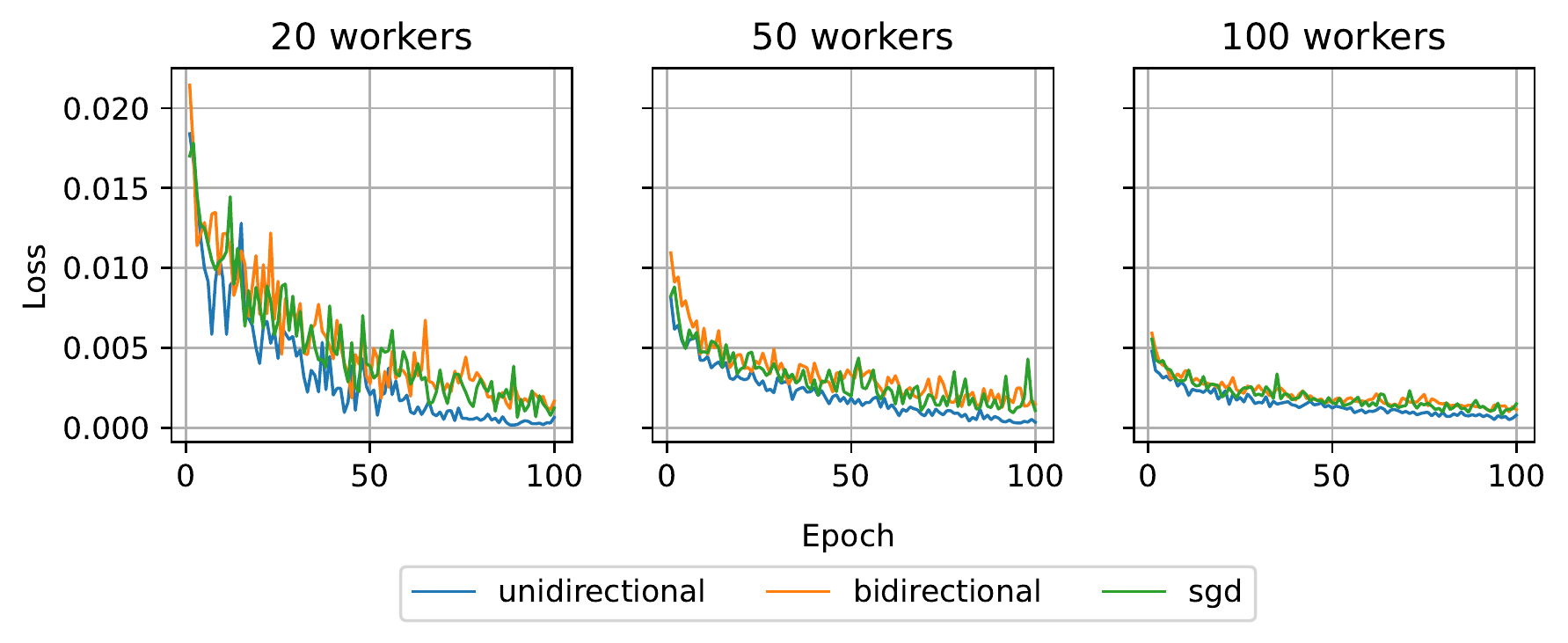}
    \caption{Comparison of training loss from training CNN model on Fashion MNIST data for unidirectional top$_{K}$,  bidirectional top$_{K}$ and vanilla SGD. Trained on batch size 10. $K_\text{downlink} = K_\text{uplink} \approx 0.001 d$ for bidirectional. $K_\text{uplink} \approx 0.001 d$ for unidirectional.}
    \label{fig:fmnist_cnn_train_loss}
\end{figure}

% \FloatBarrier

% \subsection{MNIST MLP Model}
\begin{figure}[ht]
    \centering
    \includegraphics[width=1\textwidth]{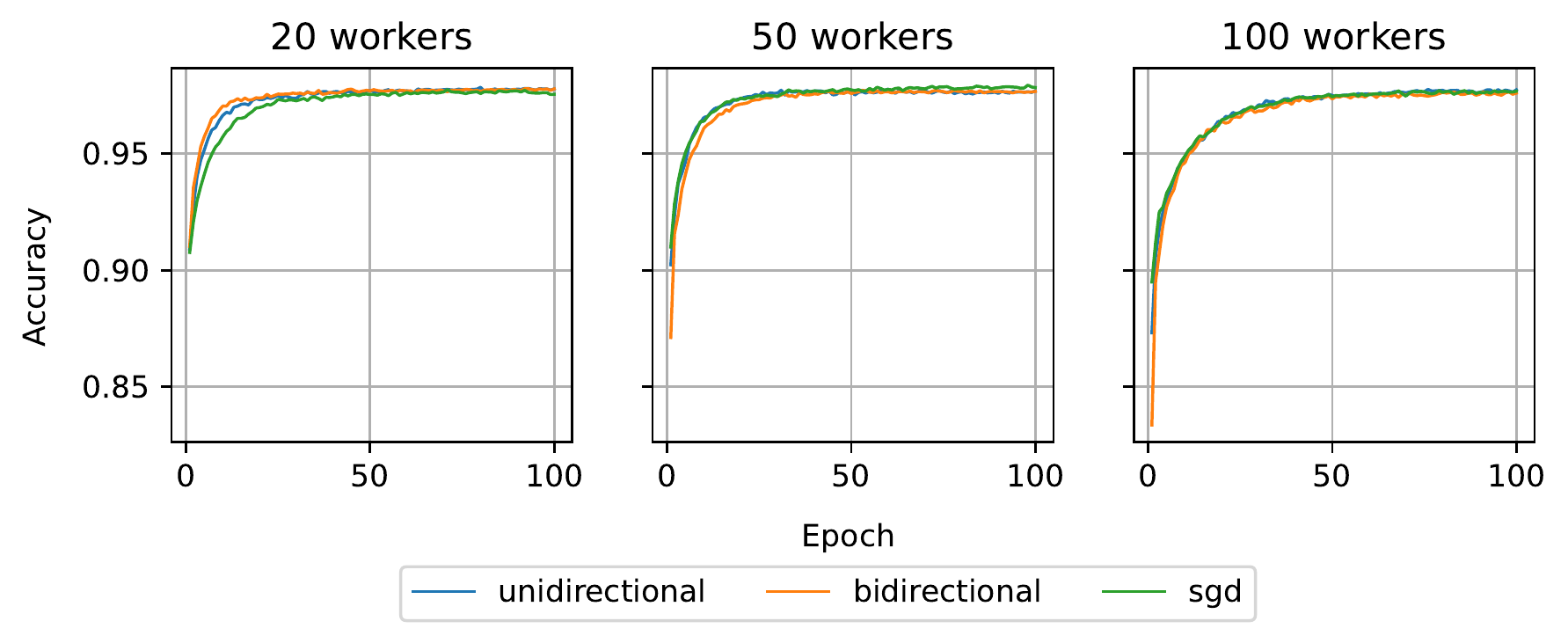}
    \caption{Comparison of testing accuracy from training MLP model on MNIST data for unidirectional top$_{K}$,  bidirectional top$_{K}$ and vanilla SGD. Trained on batch size 10. $K_\text{downlink} = K_\text{uplink} \approx 0.001 d$ for bidirectional. $K_\text{uplink} \approx 0.001 d$ for unidirectional.}
    \label{fig:mnist_mlp_test_accuracy}
\end{figure}

\begin{figure}[ht]
    \centering
    \includegraphics[width=1\textwidth]{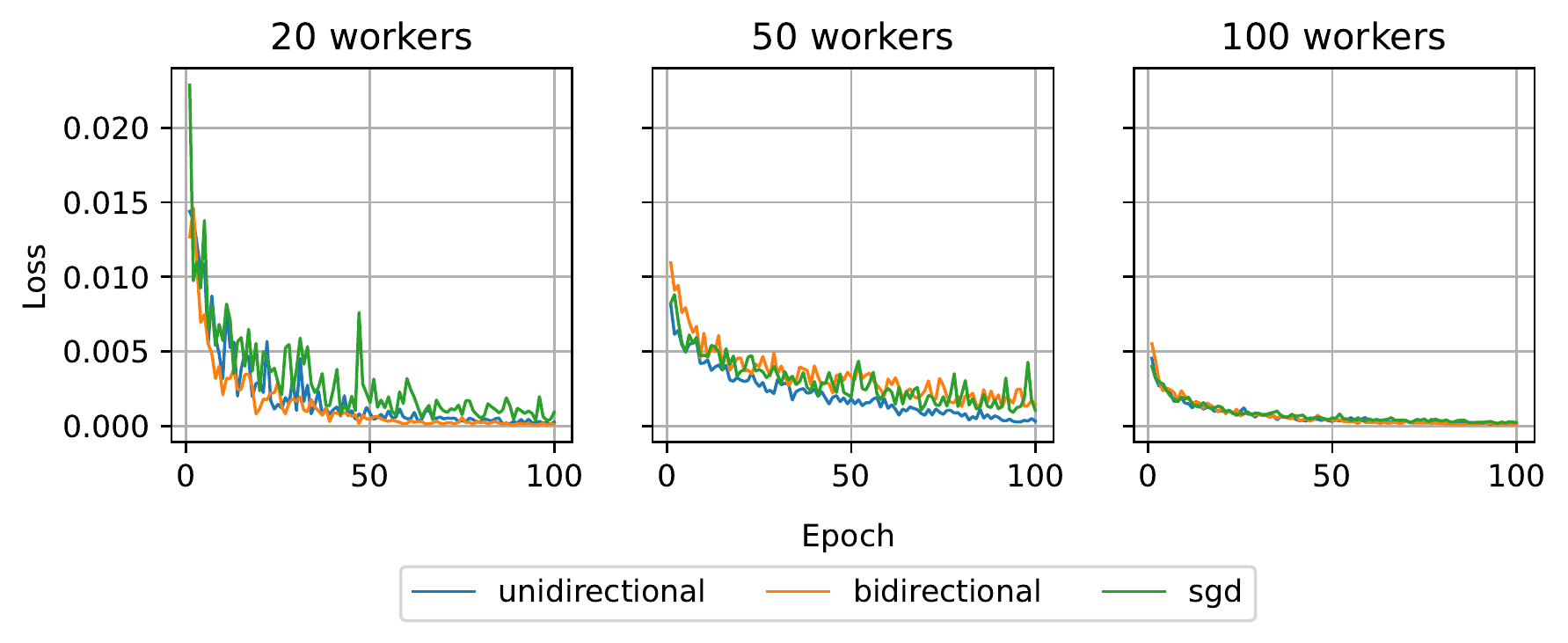}
    \caption{Comparison of training loss from training MLP model on MNIST data for unidirectional top$_{K}$,  bidirectional top$_{K}$ and vanilla SGD. Trained on batch size 10. $K_\text{downlink} = K_\text{uplink} \approx 0.001 d$ for bidirectional. $K_\text{uplink} \approx 0.001 d$ for unidirectional.}
    \label{fig:mnist_mlp_train_loss}
\end{figure}

% \FloatBarrier
% \subsection{MNIST CNN Model}

\begin{figure}[ht]
    \centering
    \includegraphics[width=1\textwidth]{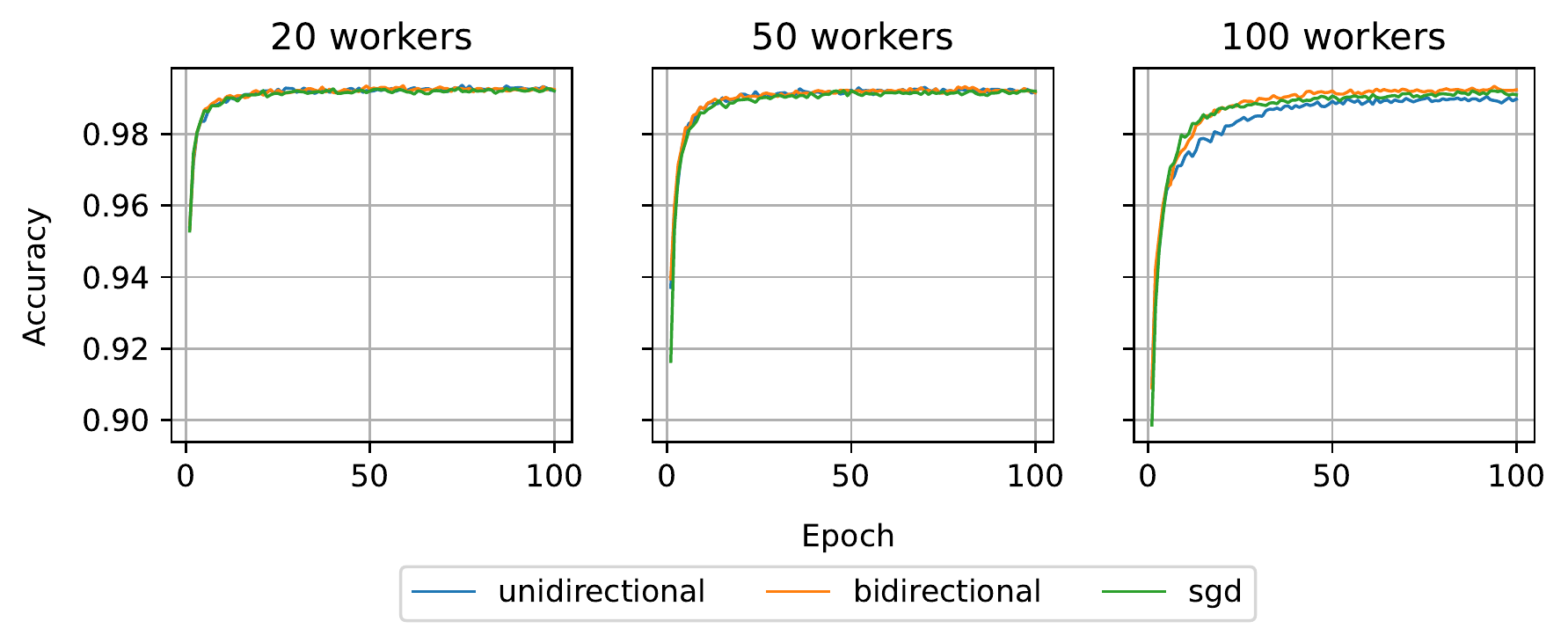}
    \caption{Comparison of testing accuracy from training CNN model on MNIST data for unidirectional top$_{K}$,  bidirectional top$_{K}$ and vanilla SGD. Trained on batch size 10. $K_\text{downlink} = K_\text{uplink} \approx 0.001 d$ for bidirectional. $K_\text{uplink} \approx 0.001 d$ for unidirectional.}
    \label{fig:mnist_cnn_test_accuracy}
\end{figure}

\begin{figure}[ht]
    \centering
    \includegraphics[width=1\textwidth]{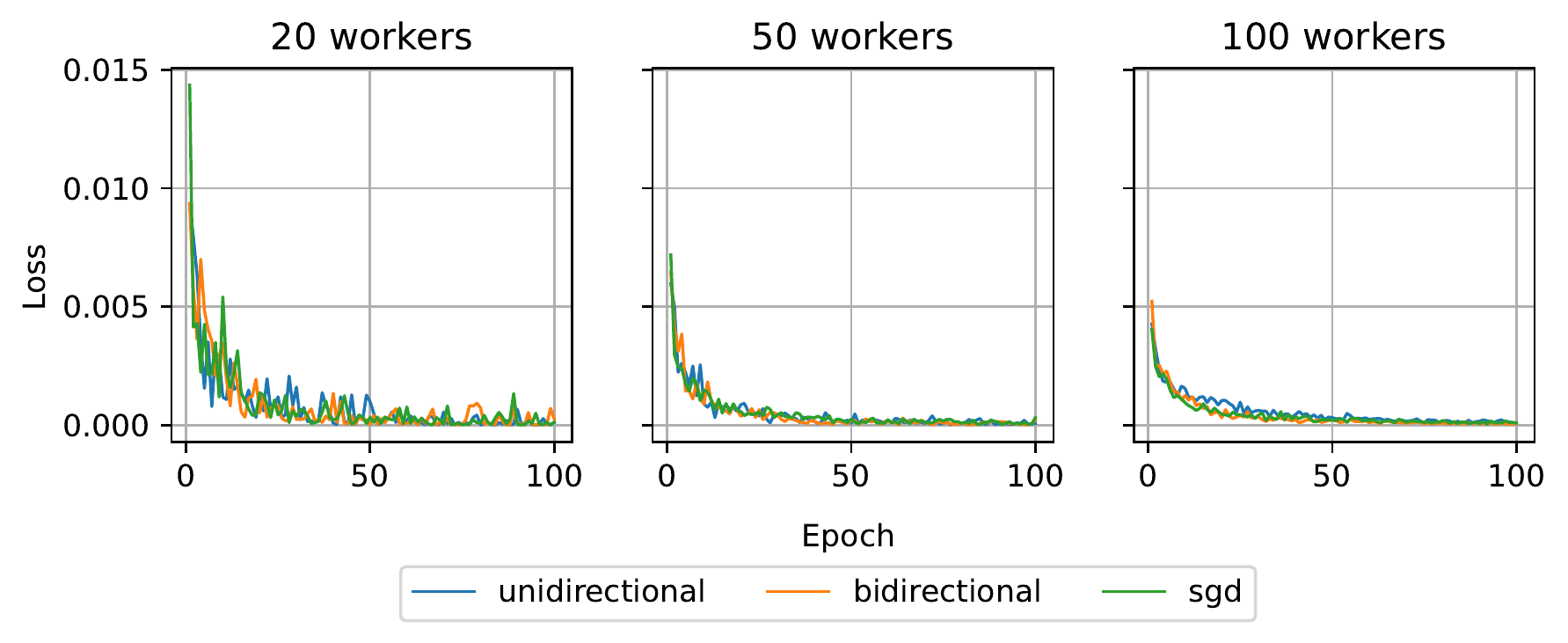}
    \caption{Comparison of training loss from training CNN model on MNIST data for unidirectional top$_{K}$,  bidirectional top$_{K}$ and vanilla SGD. Trained on batch size 10. $K_\text{downlink} = K_\text{uplink} \approx 0.001 d$ for bidirectional. $K_\text{uplink} \approx 0.001 d$ for unidirectional.}
    \label{fig:mnist_cnn_train_loss}
\end{figure}

\FloatBarrier

\subsubsection{\texorpdfstring{$\rho$ vs $\hat \rho$ Values}{Lg}}
\label{appendix:rho}
\begin{figure}[ht]
    \centering
    \includegraphics[width=1\textwidth]{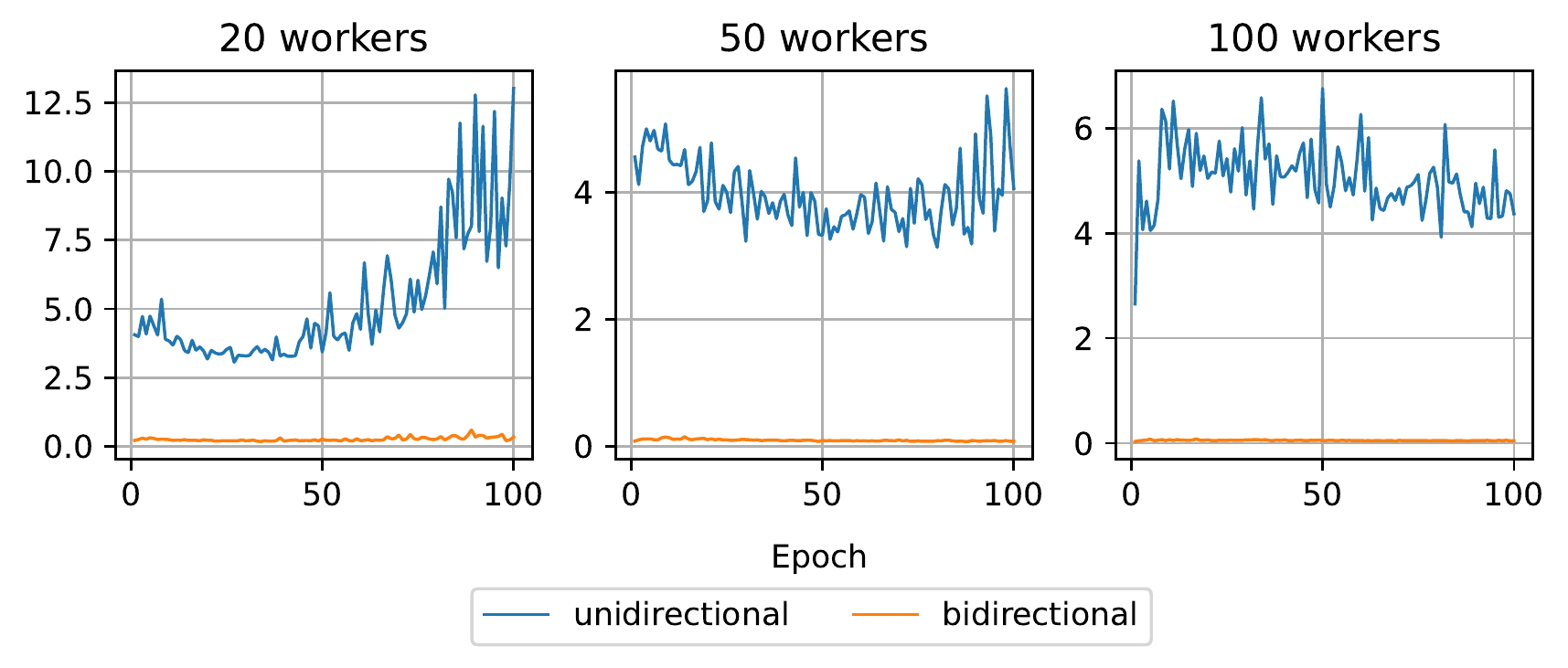}
    \caption{Comparison of $\rho$ vs $\hat \rho$ values from training MLP model on Fashion MNIST data for unidirectional top$_{K}$ and bidirectional top$_{K}$. Models trained on batch size 10. $K_\text{downlink} = K_\text{uplink} \approx 0.001 d$ for bidirectional. $K_\text{uplink} \approx 0.001 d$ for unidirectional.}
    \label{fig:fmnist_mlp_rho}
\end{figure}

\begin{figure}[ht]
    \centering
    \includegraphics[width=1\textwidth]{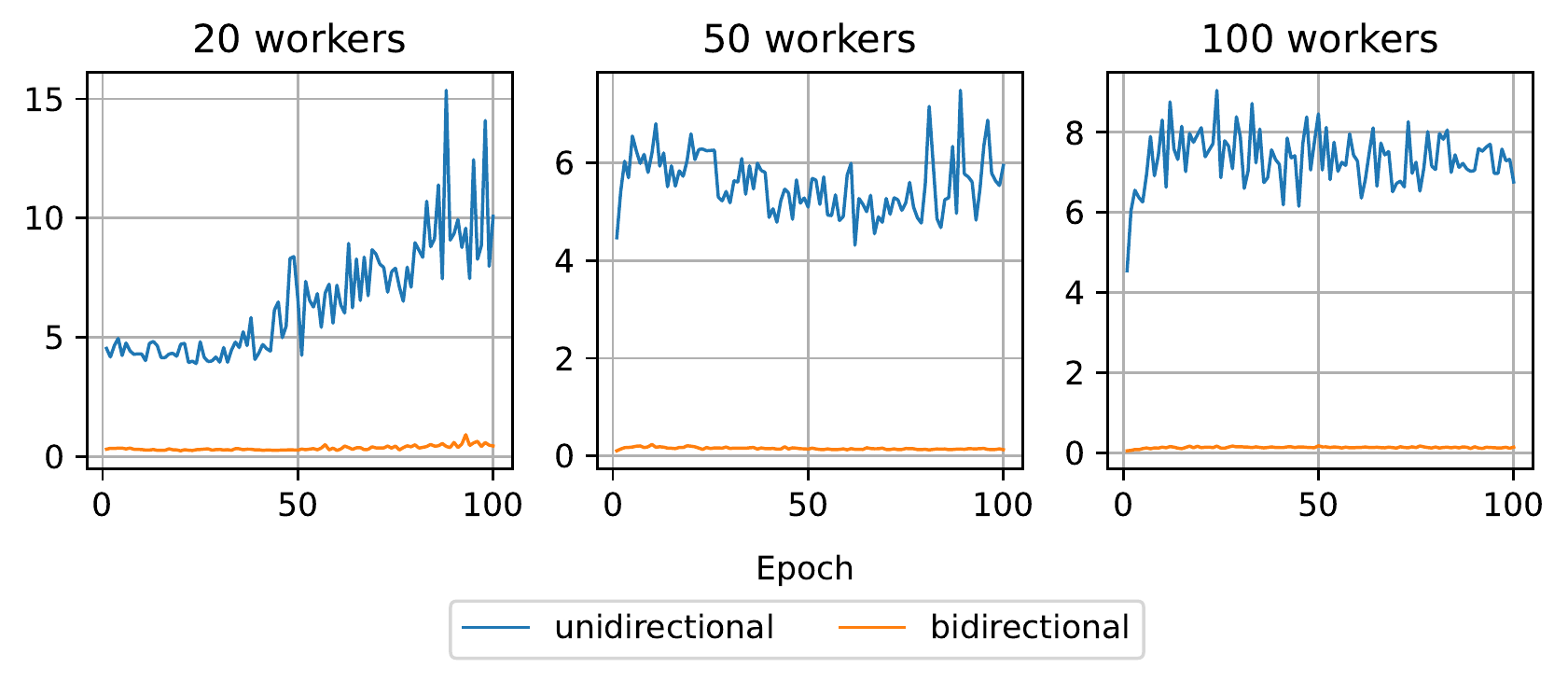}
    \caption{Comparison of $\rho$ vs $\hat \rho$ values from training CNN model on Fashion MNIST data for unidirectional top$_{K}$ and bidirectional top$_{K}$. Models trained on batch size 10. $K_\text{downlink} = K_\text{uplink} \approx 0.001 d$ for bidirectional. $K_\text{uplink} \approx 0.001 d$ for unidirectional.}
    \label{fig:fmnist_cnn_rho}
\end{figure}

\begin{figure}[ht]
    \centering
    \includegraphics[width=1\textwidth]{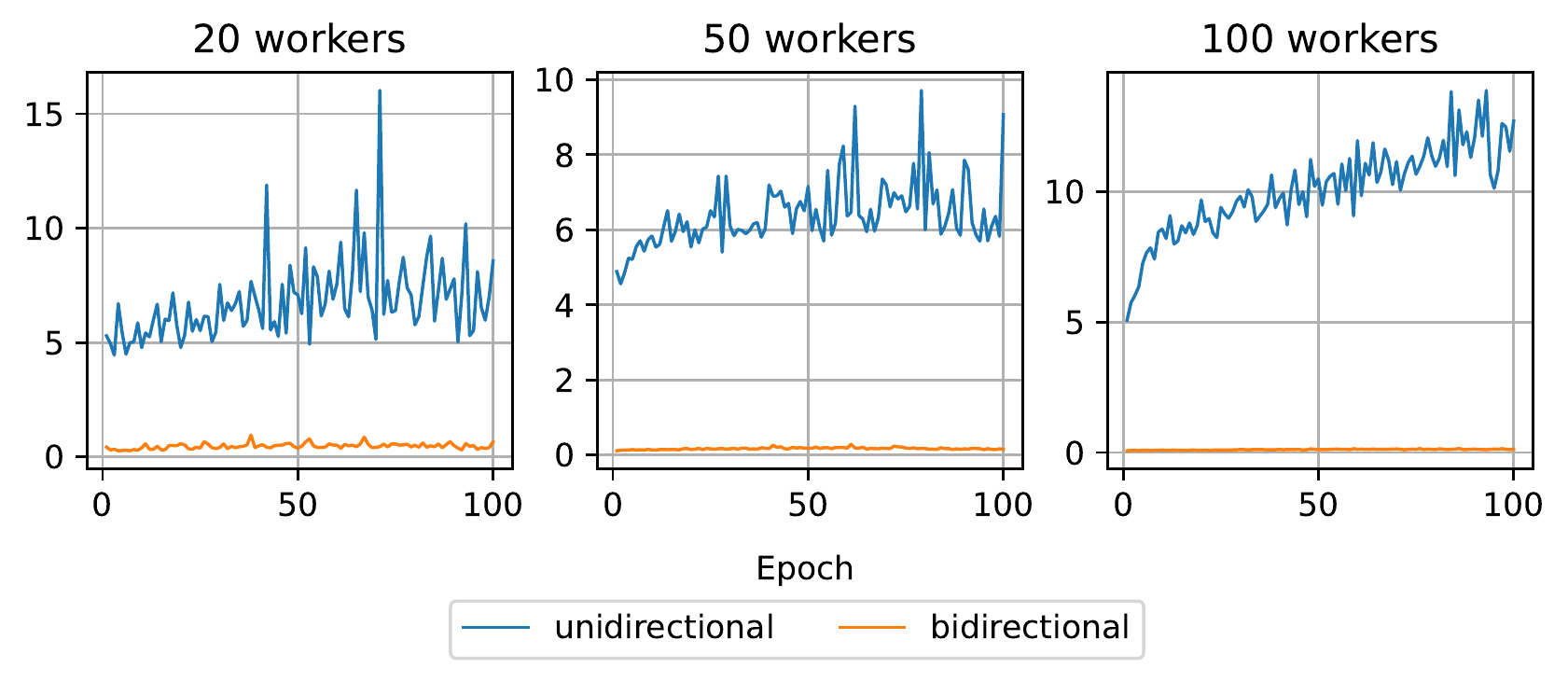}
    \caption{Comparison of $\rho$ vs $\hat \rho$ values from MLP Model trained on MNIST data for unidirectional top$_{K}$ and bidirectional top$_{K}$. Models trained on batch size 10. $K_\text{downlink} = K_\text{uplink} \approx 0.001 d$ for bidirectional. $K_\text{uplink} \approx 0.001 d$ for unidirectional.}
    \label{fig:mnist_mlp_rho}
\end{figure}

\begin{figure}[ht]
    \centering
    \includegraphics[width=1\textwidth]{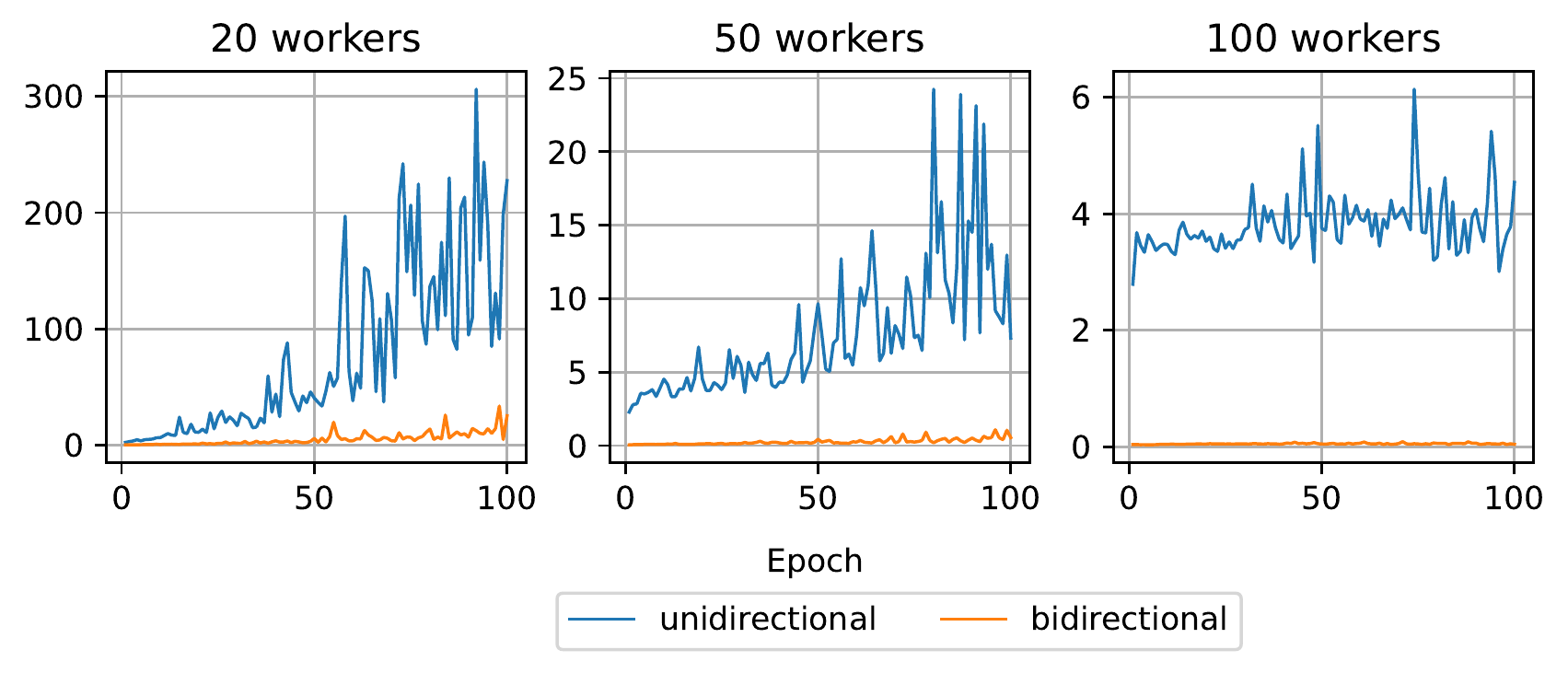}
    \caption{Comparison of $\rho$ vs $\hat \rho$ values from CNN Model trained on MNIST data for unidirectional top$_{K}$ and bidirectional top$_{K}$. Models trained on batch size 10. $K_\text{downlink} = K_\text{uplink} \approx 0.001 d$ for bidirectional. $K_\text{uplink} \approx 0.001 d$ for unidirectional.}
    \label{fig:mnist_cnn_rho}
\end{figure}

\FloatBarrier

\subsubsection{\texorpdfstring{$1 - \gamma$ values}{Lg}}
\label{appendix:gamma_values}

\begin{figure}[ht]
    \centering
    \includegraphics[width=0.6\textwidth]{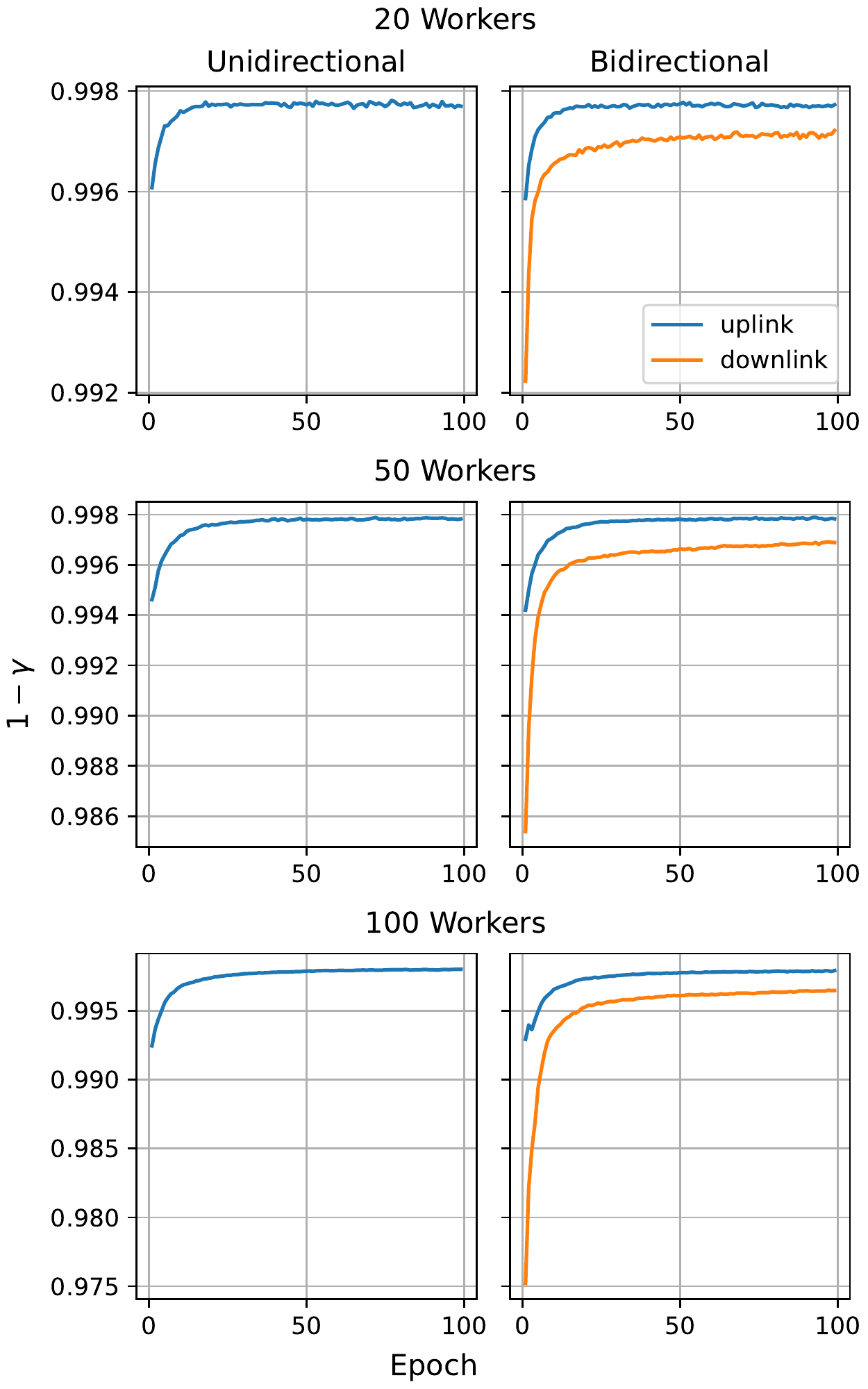}
    \caption{Largest $1 - \gamma$ value in each epoch of a MLP model trained with unidirectional and bidirectional top$_K$ SGD on Fashion MNIST dataset.}
    \label{fig:fmnist_mlp_delta}
\end{figure}

\begin{figure}[ht]
    \centering
    \includegraphics[width=0.6\textwidth]{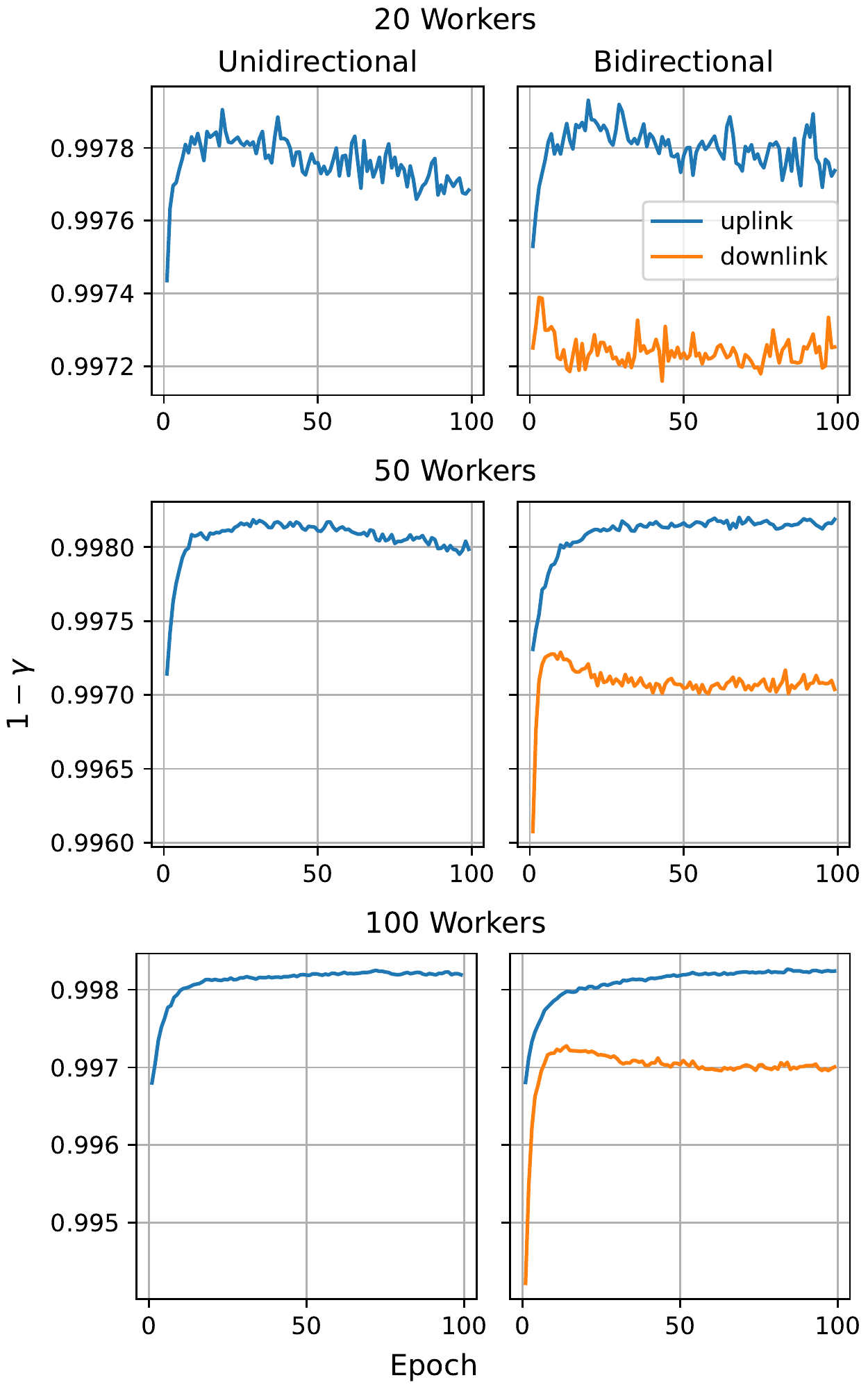}
    \caption{Largest $1 - \gamma$ value in each epoch of a CNN model trained with unidirectional and bidirectional top$_K$ SGD on a Fashion MNIST dataset.}
\end{figure}

\begin{figure}[ht]
    \centering
    \includegraphics[width=0.6\textwidth]{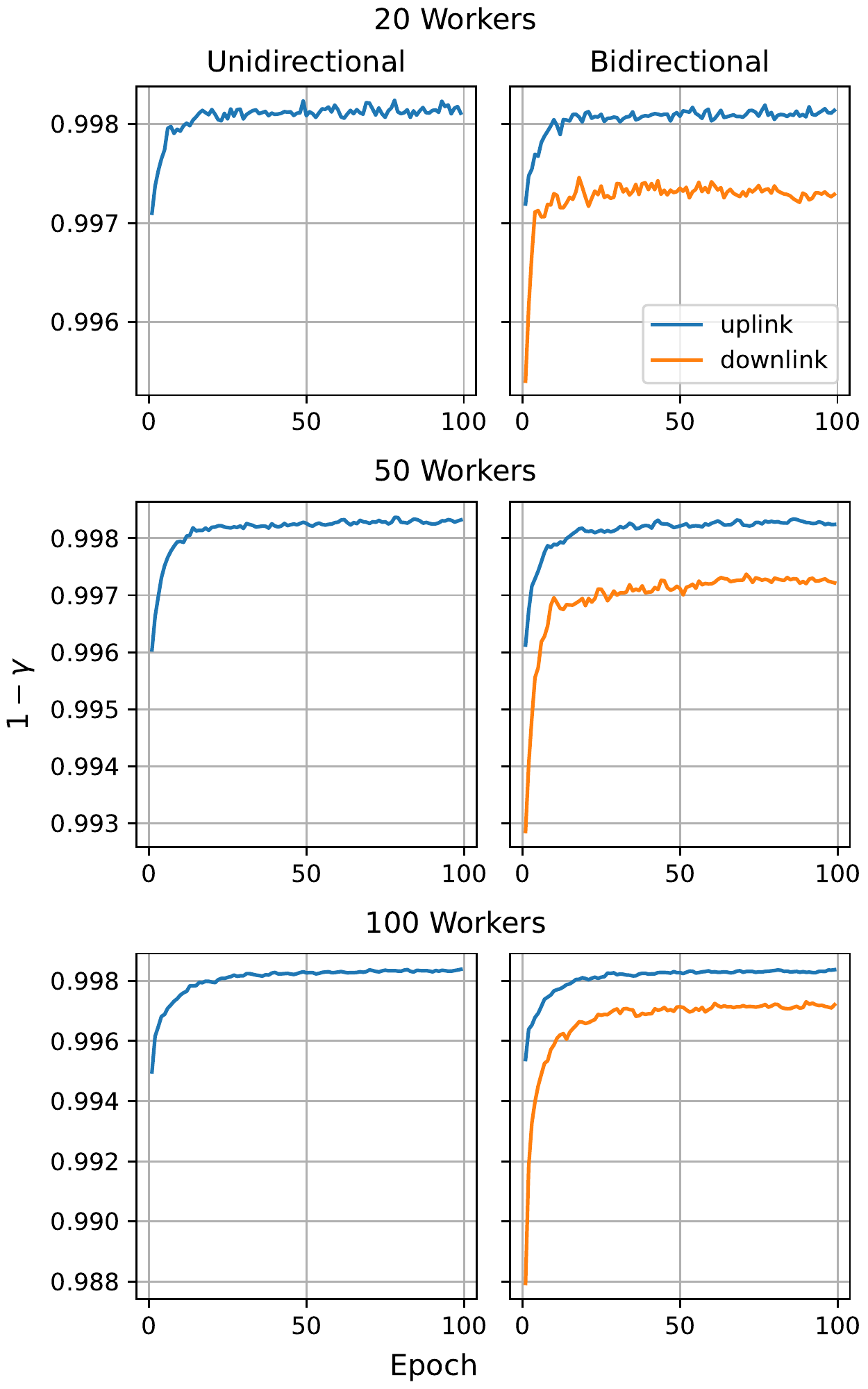}
    \caption{Largest $1 - \gamma$ value in each epoch of a MLP model trained with unidirectional and bidirectional top$_K$ SGD on a MNIST dataset.}
\end{figure}

\begin{figure}[ht]
    \centering
    \includegraphics[width=0.6\textwidth]{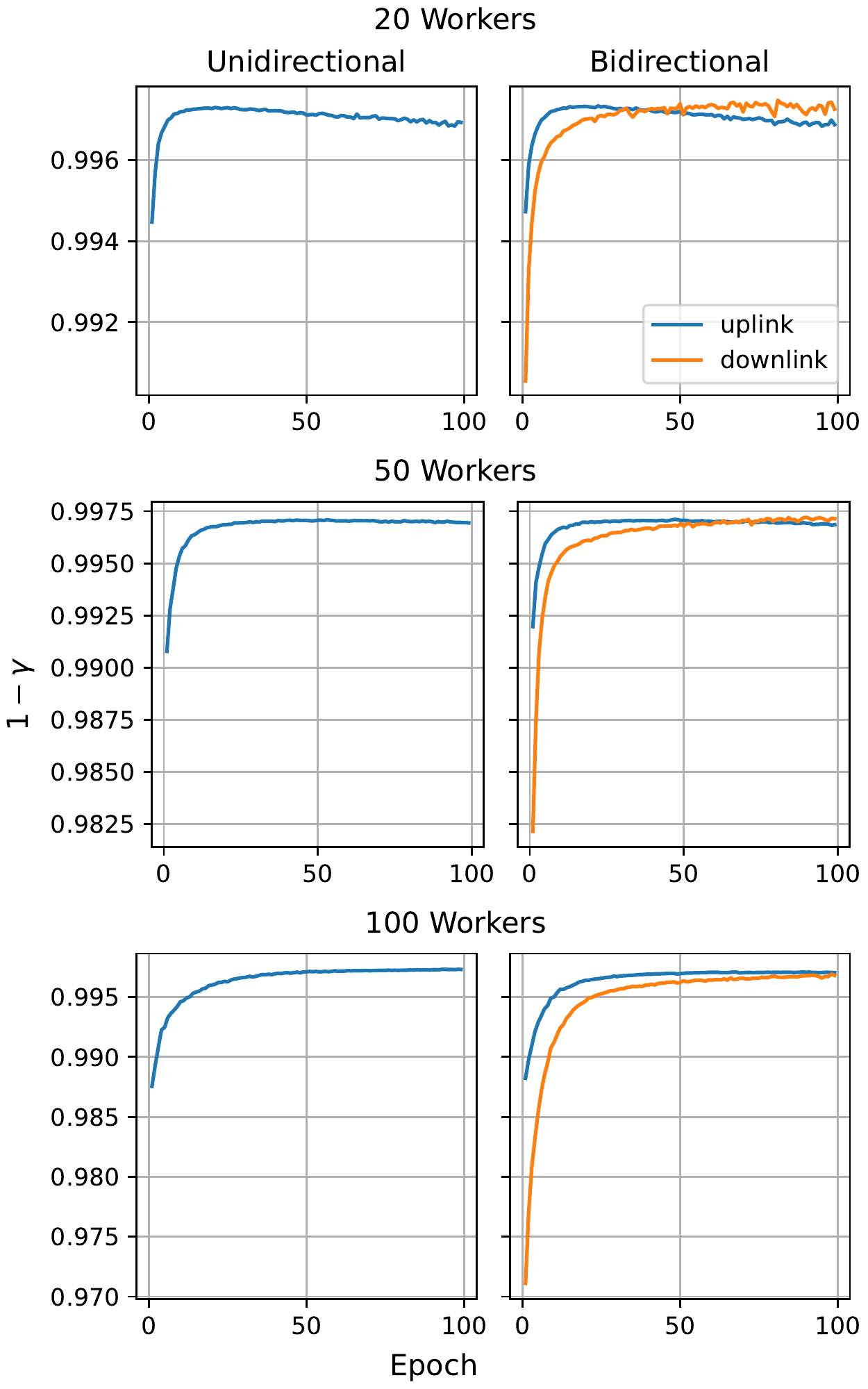}
    \caption{Largest $1 - \gamma$ value in each epoch of a CNN model trained with unidirectional and bidirectional top$_K$ SGD on a MNIST dataset.}
    \label{fig:mnist_cnn_delta}
\end{figure}

\FloatBarrier

\subsubsection{Fraction of Non-Zero Indices in Gradient After Server Aggregation}
\label{appendix:non-zero}
\begin{figure}[ht]
    \centering
    \includegraphics[width=1\textwidth]{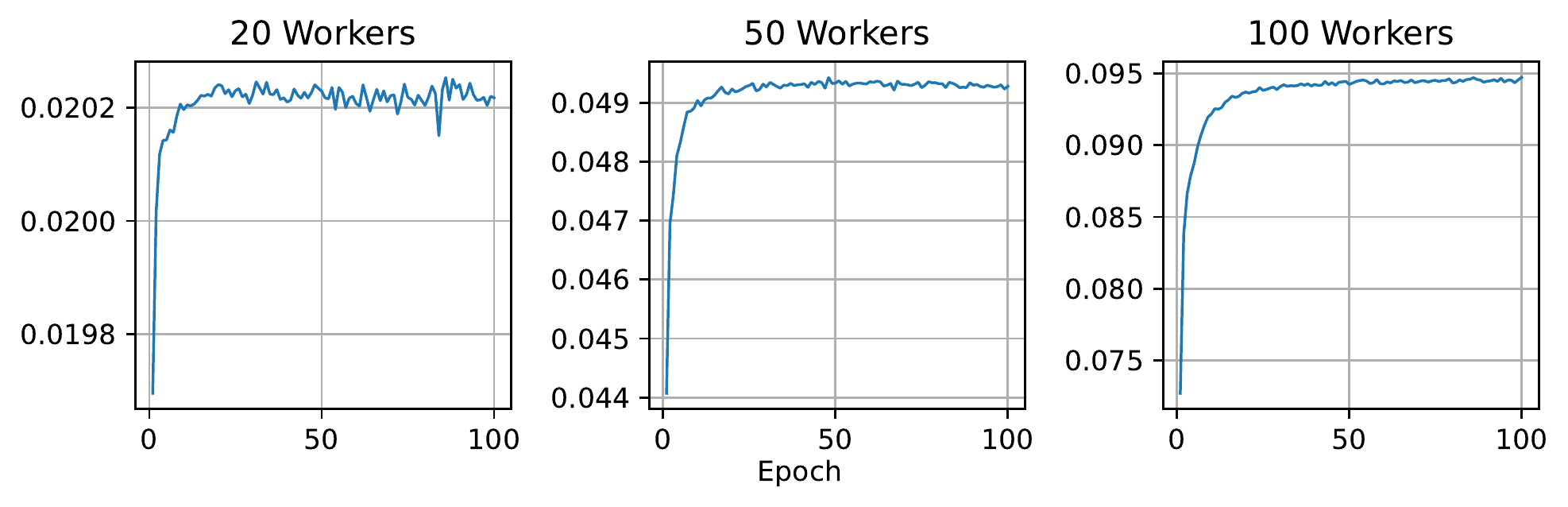}
    \caption{Percent of non-zero indices after aggregating sparse gradients from workers. Trained on a CNN Model using Fashion MNIST dataset. $K_\text{uplink} = K_\text{downlink} \approx  0.001d$.}
    \label{fig:fmnist_cnn_compression_upstream}
\end{figure}

\begin{figure}[ht]
    \centering
    \includegraphics[width=1\textwidth]{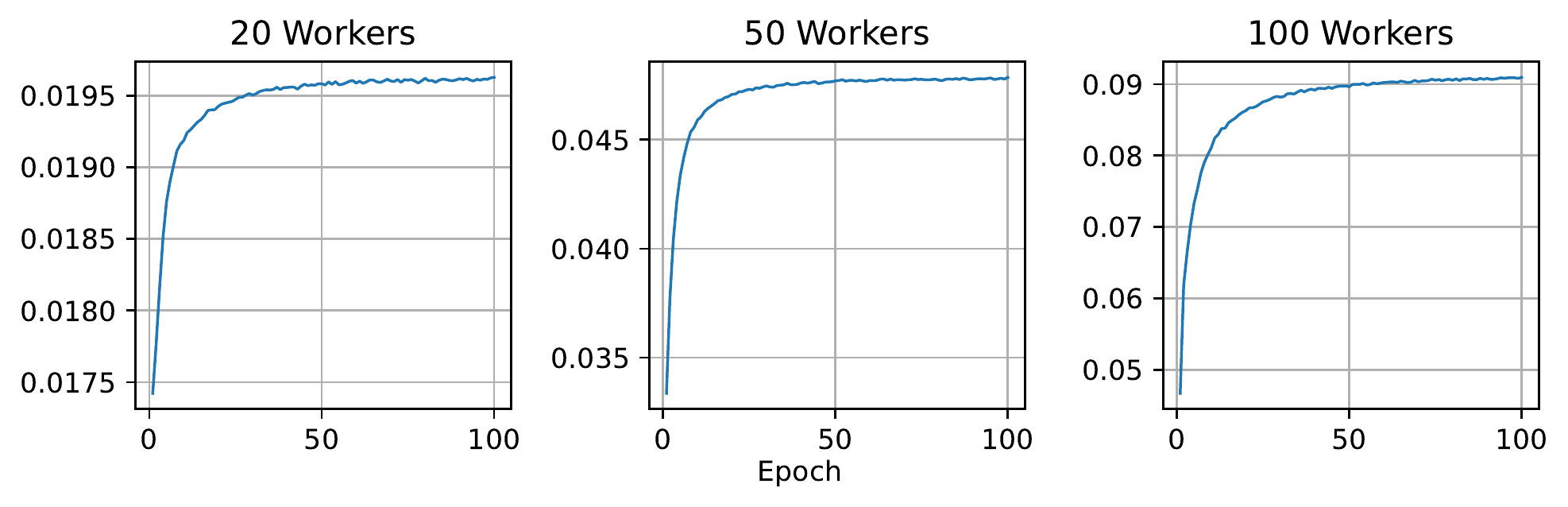}
    \caption{Percent of non-zero indices after aggregating sparse gradients from workers. Trained on a MLP Model using MNIST dataset. $K_\text{uplink} = K_\text{downlink} \approx  0.001d$.}
    \label{fig:mnist_mlp_compression_upstream}
\end{figure}

\begin{figure}[ht]
    \centering
    \includegraphics[width=1\textwidth]{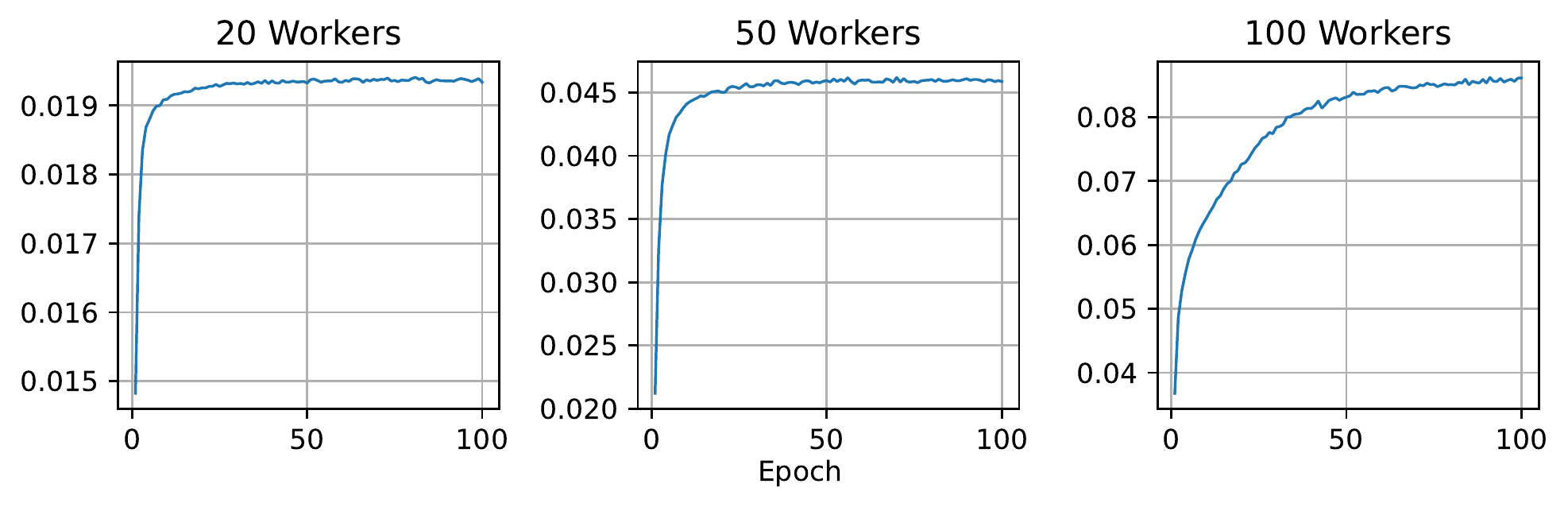}
    \caption{Fraction of non-zero indices after aggregating sparse gradients from workers. Trained on a CNN Model using MNIST dataset. $K_\text{uplink} = K_\text{downlink} \approx  0.001d$.}
    \label{fig:mnist_cnn_compression_upstream}
\end{figure}
\end{document}